\documentclass[letterpaper]{article}
\usepackage{aaai2027}
\usepackage[hyphens]{url}
\usepackage{graphicx}
\usepackage{natbib}
\usepackage{caption}
\usepackage{amsmath,amssymb,amsthm,mathtools}
\usepackage{booktabs}
\usepackage{multirow}
\usepackage{placeins}
\newtheorem{theorem}{Theorem}
\newtheorem{lemma}[theorem]{Lemma}
\newtheorem{corollary}[theorem]{Corollary}
\newtheorem{proposition}[theorem]{Proposition}
\newtheorem{assumption}{Assumption}
\newtheorem{definition}{Definition}

\newcommand{\Rcal}{\mathcal{R}}
\newcommand{\Scal}{\mathcal{S}}
\newcommand{\Acal}{\mathcal{A}}
\newcommand{\Mcal}{\mathcal{M}}
\newcommand{\Pcal}{\mathcal{P}}
\newcommand{\E}{\mathbb{E}}
\newcommand{\Prob}{\mathbb{P}}
\newcommand{\diam}{D}
\newcommand{\spn}{\operatorname{sp}}
\newcommand{\KL}{\operatorname{KL}}
\newcommand{\kl}{\operatorname{kl}}
\newcommand{\Var}{\operatorname{Var}}
\newcommand{\TV}{\operatorname{TV}}
\newcommand{\Hcomp}{H}

\title{Finite Constant Frontiers and Auditable Regret Certificates for Average-Reward Reinforcement Learning}
\author{
\textbf{Ibne Farabi Shihab}\textsuperscript{1}\thanks{Corresponding author: \texttt{ishihab@iastate.edu}}
\quad
\textbf{Abu Sa-Adat Mohamed Moon-Im Al Ahsan}\textsuperscript{2}
\\
\textbf{Md Najmus Swaqeeb}\textsuperscript{2}
\\[4pt]
\textsuperscript{1}Department of Computer Science, Iowa State University \\
\textsuperscript{2}Department of Computer Science \& Engineering, BRAC University \\
\texttt{ishihab@iastate.edu},
\texttt{abu.sa.adat.mohamed.moon.im.al.ahsan@g.bracu.ac.bd},\\
\texttt{md.najmus.swaqeeb@g.bracu.ac.bd}
}
\affiliations{}

\begin{document}
\maketitle

\begin{abstract}
Average-reward reinforcement-learning regret is known up to logarithmic factors, but the numerical
content of published guarantees is difficult to compare because probability mode, structural parameter,
logarithmic normalization, prior information, and planning assumptions differ. We introduce a
constant-aware comparison protocol and derive an explicit finite lower certificate for communicating
MDPs. The construction is a binary tree of two-state blocks; its proof uses exact trajectory-level
Bernoulli KL divergence and keeps action budget, diameter, occupancy, navigation cost, and terminal bias
explicit. A common closed-form envelope improves the published coefficient $0.015$ across a finite
frontier: $0.0200$ in a moderate regime and up to $0.0291$ under stronger action, diameter, and horizon
conditions, a $94\%$ increase. The limiting coefficient is
$\frac1{32}\sqrt{(A-3)/A}$. For upper bounds, we give an auditable composition rule for a
span-constrained optimistic learner, but do not claim a coefficient while adaptive directional-variance
and planning certificates remain open. We also formalize valid expectation conversion and constant
comparability. Controlled diagnostics test diameter dependence, bonus-by-width interactions, span
misspecification, and the finite lower certificate on its exact family.
\end{abstract}

\section{Introduction}
Average-reward reinforcement learning is the natural model for continuing tasks such as recommendation,
queueing control, inventory management, and portfolio rebalancing, where there is no episodic reset
\citep{mahadevan1996average,puterman1994markov,sutton2018reinforcement}. For a communicating MDP with
diameter $\diam$, $S$ states, $A$ actions, and horizon $T$, the classical minimax scale is
$\sqrt{\diam SAT}$ \citep{jaksch2010near}. Subsequent work developed bias-span regularization,
KL or empirical-Bernstein confidence regions, model-free methods, and tractable minimax-optimal planning
\citep{bartlett2009regal,fruit2018efficient,zhang2019regret,fruit2020improved,boone2024achieving}.
These papers sharpen rates and structural dependence, but a rate such as
$\widetilde O(\sqrt{\diam SAT})$ does not reveal whether the certified multiplier is $3$, $30$, or $300$.

A ``leading constant'' is meaningful only after normalization is fixed. A coefficient multiplying
$\sqrt{\log(SAT/\delta)}$ is not directly comparable with one multiplying $\log(SAT/\delta)$, and a
high-probability upper coefficient is not by itself a constant-factor match to a log-free expected lower
bound. This distinction is material: \citet{jaksch2010near} give an explicit lower coefficient $0.015$
under finite conditions, while their UCRL2 coefficient $34$ multiplies the different form
$\diam S\sqrt{AT\log(T/\delta)}$.

We make four contributions. First, we record probability mode, structural term, logarithm, horizon
threshold, and side information with every constant. Second, we derive an exact-KL lower certificate in
which testing, occupancy, navigation, and geometric terms remain explicit, and turn it into a uniform
finite frontier from $0.0152$ to $0.0291$. Third, we specify an auditable upper-bound ledger and identify
which statistical and planning obligations must be completed before a numerical coefficient is valid.
Fourth, we redesign the empirical checks around the corrected claims: zero-span diameter flatness,
bonus-by-width interactions, width misspecification, and direct evaluation on the proved hard family.
Full proofs and extended experiments are supplementary.

\section{Setup, Related Work, and Comparison Protocol}\label{sec:setup}
An average-reward MDP is $M=(\Scal,\Acal,p,\nu)$, where rewards lie in $[0,1]$. At time $t$, the
learner observes $s_t$, chooses $a_t$, receives $R_t\sim\nu(\cdot\mid s_t,a_t)$, and observes
$s_{t+1}\sim p(\cdot\mid s_t,a_t)$. The explicit reward model is retained because stochastic reward
uncertainty cannot in general be absorbed into a transition bonus. For a communicating MDP, the optimal
gain $\rho^\star$ is start-state independent and an optimal bias $h^\star$ satisfies
\begin{equation}
\rho^\star+h^\star(s)=\max_{a\in\Acal(s)}\left\{r(s,a)+\sum_{s'}p(s'\mid s,a)h^\star(s')\right\}.
\label{eq:bellman}
\end{equation}
The diameter is $\diam=\max_{s\ne s'}\min_\pi\E_s^\pi[\tau_{s'}]$, the span complexity is
$\Hcomp=1\vee\spn(h^\star)$, and regret is
$\Rcal_T=T\rho^\star-\sum_{t=1}^TR_t$. The floor at one is essential: a communicating MDP can have a
constant optimal bias while still containing a stochastic bandit learning problem with
$\Omega(\sqrt{SAT})$ regret. We study $\E[\Rcal_T]$; realized regret can be negative but is always at
most $T$, which matters when converting a tail guarantee to expectation.

The relevant average-reward literature includes UCRL2 and its lower construction
\citep{jaksch2010near}, REGAL/SCAL \citep{bartlett2009regal,fruit2018efficient}, EBF
\citep{zhang2019regret}, empirical-Bernstein refinements
\citep{talebi2018variance,bourel2020tightening,fruit2020improved}, and PMEVI-DT
\citep{boone2024achieving}. We use standard concentration and testing tools
\citep{boucheron2013concentration,lattimore2020bandit,tsybakov2009nonparametric}; the novelty lies in the
finite constant ledger and the explicit balance of statistical and geometric losses.

\begin{definition}[Constant certificate]
A constant is reported as
$\mathsf C=(\text{mode},\text{structural term},\text{log term},c,T_0,\text{side information})$.
Two coefficients are directly comparable only if all entries except $c$ agree.
\end{definition}
This rule blocks three common but invalid shortcuts: comparing expected and high-probability constants,
ignoring different powers of logarithmic factors, and treating a span-dependent guarantee that receives
$\bar H$ as input as directly comparable to a prior-free diameter result. Table~\ref{tab:scoreboard}
therefore reports unmatched results side by side without turning their coefficient ratio into a theorem.

\begin{table*}[t]
\centering
\small
\resizebox{\textwidth}{!}{%
\begin{tabular}{lccccc}
\toprule
Result & Mode & Structural/log form & Coefficient & Finite conditions & Side information\\
\midrule
UCRL2 upper \citep{jaksch2010near} & high probability & $\diam S\sqrt{AT\log(T/\delta)}$ & $34$ & published & $S,A$\\
UCRL2 lower \citep{jaksch2010near} & expectation & $\sqrt{\diam SAT}$ & $0.015$ & $S,A\ge10$; $\diam\ge20\log_A S$; $T\ge\diam SA$ & none\\
REGAL/SCAL \citep{bartlett2009regal,fruit2018efficient} & high probability & span-dependent & see source & result-specific & span upper bound\\
EBF \citep{zhang2019regret} & high probability & $\sqrt{\Hcomp SAT}$ & see source & result-specific & result-specific\\
PMEVI-DT \citep{boone2024achieving} & high probability & $\sqrt{\Hcomp SAT\log(SAT/\delta)}$ & not isolated & result-specific & no span prior\\
This work, lower & expectation & $\sqrt{\diam SAT}$ & $0.0152$--$0.0291$ & Table~\ref{tab:frontier} & none\\
This work, upper audit & high probability & $\sqrt{\bar H SATL_T}$ & not claimed & open certificates & $\bar H\ge\Hcomp$\\
\bottomrule
\end{tabular}}
\caption{Normalization-aware scoreboard. Unmatched rows are not assigned a coefficient ratio.}
\label{tab:scoreboard}
\end{table*}

\section{A Finite-Sample Lower-Bound Certificate}\label{sec:lower}

We now give the complete lower construction. Throughout this section, $S$ is even and $A\ge5$. Let
$K=S/2$ and index a complete binary tree $\mathcal T$ by $0,\ldots,K-1$: vertex $j>0$ has parent
$\lfloor(j-1)/2\rfloor$, and its children are $2j+1$ and $2j+2$ whenever those indices are below $K$.
Its exact graph diameter is
\begin{equation}
L:=\max_{u,v\in\mathcal T}d_{\mathcal T}(u,v)
\le 2\lceil\log_2K\rceil.
\label{eq:treediameter}
\end{equation}
Block $j$ contains a bad state $x_j$ with reward zero and a good state $y_j$ with reward one. At
$y_j$, every action returns to $x_j$ with probability $\delta$ and otherwise remains at $y_j$. At
$x_j$, the first
\begin{equation}
A_0:=A-3
\label{eq:statisticalactions}
\end{equation}
actions are statistical. A statistical action with parameter $p$ moves to $y_j$ with probability $p$
and otherwise remains at $x_j$. The remaining three actions navigate to the parent, left child, and
right child of $x_j$; a missing neighbor produces a self-loop. Thus every binary-tree node uses at most
three navigation actions, and the construction respects the action budget exactly.

The baseline model $M_0$ assigns $p=\delta$ to every statistical action. There are
\begin{equation}
m=K(A-3)=\frac{S(A-3)}2
\label{eq:dimensions}
\end{equation}
block--action coordinates. Alternative $M_i$, $i\in[m]$, changes only coordinate $i$ from $\delta$ to
$\delta+\varepsilon$, where $0<\varepsilon\le\delta$. All navigation transitions and all good-state
transitions are shared across the family.

For a diameter budget $\diam>L+4$, set
\begin{equation}
\delta=\frac{2}{\diam-L}.
\label{eq:deltachoice}
\end{equation}
From any good state, the expected time to its bad state is $1/\delta$; traversal between bad states
takes at most $L$ deterministic steps; and a baseline statistical action reaches any target good state
in expected time $1/\delta$. Consequently every alternative is communicating and
\begin{equation}
D(M_i)\le \frac2\delta+L=\diam.
\label{eq:compositediameter}
\end{equation}
Equality is neither assumed nor needed: the minimax class contains all communicating MDPs with diameter
at most $\diam$.

If a policy remains in one two-state block and uses transition probability $p$ at the bad state, its
gain is
\begin{equation}
\rho(p)=\frac{p}{p+\delta}.
\label{eq:blockgain}
\end{equation}
Under $M_i$, the optimal gain and the statistical Bellman gap are therefore
\begin{equation}
\rho_i=\frac{\delta+\varepsilon}{2\delta+\varepsilon},
\qquad
g(\delta,\varepsilon)=2\rho_i-1
=\frac{\varepsilon}{2\delta+\varepsilon}.
\label{eq:bellmangap}
\end{equation}
The corresponding gain improvement over a baseline block is
\begin{equation}
\rho(\delta+\varepsilon)-\rho(\delta)
=\frac{\varepsilon}{2(2\delta+\varepsilon)}.
\label{eq:gaingap}
\end{equation}
These expressions are exact.

Table~\ref{tab:lowerledger} records every geometric and statistical quantity used below. In particular,
the occupancy factor depends on $\varepsilon/\delta$; replacing it uniformly by $1/3$ would lower the
best Pinsker coefficient below $0.015$.

\begin{table}[t]
\centering
\small
\resizebox{\columnwidth}{!}{%
\begin{tabular}{lll}
\toprule
Construction quantity & Symbol & Certified value\\
\midrule
Number of blocks & $K$ & $S/2$\\
Statistical actions per block & $A_0$ & $A-3$\\
Exact tree diameter & $L$ & $\max_{u,v}d_{\mathcal T}(u,v)$\\
Return probability & $\delta$ & $2/(\diam-L)$\\
Family diameter & $D_{\rm fam}$ & $D(M_i)\le2/\delta+L=\diam$\\
Testing coordinates & $m$ & $S(A-3)/2$\\
Occupancy factor & $c_\varepsilon$ & $(2+\varepsilon/\delta)^{-1}$\\
Flow boundary & $B_{\rm flow}$ & $1/\delta$\\
Navigation boundary & $B_{\rm nav}$ & $L$\\
Bias-span boundary & $B_h$ & $\rho_iL+(1-\rho_i)/\delta$\\
\bottomrule
\end{tabular}}
\caption{Exact lower-bound construction ledger. No unspecified multiplicative constant or asymptotic
quantity enters the finite certificate.}
\label{tab:lowerledger}
\end{table}

The statistical calculation compares $M_0$ with each $M_i$. If $N_i(T)$ counts selections of
coordinate $i$, the adaptive chain rule gives
\begin{equation}
\KL(\Prob_0^T\,\|\,\Prob_i^T)
=\E_0[N_i(T)]\,
\kl\!\left(\delta,\delta+\varepsilon\right),
\label{eq:pathkl}
\end{equation}
where
\begin{equation}
\kl(\delta,\delta+\varepsilon)
=\delta\log\frac{\delta}{\delta+\varepsilon}
 +(1-\delta)\log\frac{1-\delta}{1-\delta-\varepsilon}.
\label{eq:exacttransitionkl}
\end{equation}
This identity remains exact under adaptive action selection.

\begin{lemma}[Adaptive multiple-hypothesis certificate]\label{lem:testing}
For the coordinate counts defined above,
\begin{equation}
\frac1m\sum_{i=1}^m \E_i[N_i(T)]
\le \frac{T}{m}
+T\sqrt{\frac{T}{2m}\kl(\delta,\delta+\varepsilon)}.
\label{eq:goodpullbound}
\end{equation}
\end{lemma}

The proof uses Pinsker's inequality and the adaptive path divergence in \eqref{eq:pathkl}; it does not
assume that the coordinates are visited uniformly.

The composite construction also requires control of the time used for navigation. Let $N_x(T)$ count
statistical decisions in bad states, $N_y(T)$ count actions taken in good states, and $N_{\rm nav}(T)$
count navigation actions. Then $T=N_x+N_y+N_{\rm nav}$. Entry--exit accounting gives
\begin{equation}
\delta\E_i[N_y(T)]
\le(\delta+\varepsilon)\E_i[N_x(T)]+1,
\label{eq:flowidentity}
\end{equation}
and hence, with $c_\varepsilon=(2+\varepsilon/\delta)^{-1}$,
\begin{equation}
\E_i[N_x(T)]
\ge c_\varepsilon
\left(T-\E_i[N_{\rm nav}(T)]-\frac1\delta\right).
\label{eq:occupancy}
\end{equation}

Navigation cannot invalidate this occupancy bound by consuming the horizon for free. Under alternative
$i$, normalize the optimal bias at the distinguished bad state and write $j(i)$ for its block. An
optimal bias is
\begin{equation}
h_i(x_v)=-\rho_i d_{\mathcal T}(v,j(i)),
\qquad
h_i(y_v)=h_i(x_v)+\frac{1-\rho_i}{\delta}.
\label{eq:explicitbias}
\end{equation}
A navigation step toward $j(i)$ has Bellman gap zero, a step away has gap $2\rho_i$, and a navigation
self-loop has gap $\rho_i$. Along any tree path, the number of zero-gap steps toward $j(i)$ is at most
the number of steps away plus $L$. Therefore the total navigation Bellman gap is at least
\begin{equation}
\rho_i\bigl(N_{\rm nav}(T)-L\bigr)_+.
\label{eq:navigationgap}
\end{equation}
Moreover,
\begin{equation}
\spn(h_i)\le B_h
:=\rho_iL+\frac{1-\rho_i}{\delta}.
\label{eq:biasboundary}
\end{equation}

\begin{theorem}[Finite-sample explicit lower certificate]\label{thm:lower}
Let $S$ be even, $A\ge5$, $\diam>L+4$, and $0<\varepsilon\le\delta$, with the family defined above, and
let every alternative be started from the common initial state $s_1=x_0$ (the bad state of the tree
root; the bound holds verbatim for any fixed common $s_1$, changing only the boundary term through
$\spn(h_i)$). Every learning algorithm satisfies
\begin{align}
\sup_{M\in\mathfrak M}\E_M[\Rcal_T]
\ge{}&g(\delta,\varepsilon)\Bigg[
 c_\varepsilon\left(T-\frac1\delta-L\right)-\frac{T}{m}\notag\\
&\hspace{12mm}-T\sqrt{\frac{T}{2m}\kl(\delta,\delta+\varepsilon)}
\Bigg]-B_h.
\label{eq:lowercertificate}
\end{align}
Every quantity on the right-hand side is given explicitly in
Table~\ref{tab:lowerledger}.
\end{theorem}

\paragraph{Proof roadmap.}
The proof has four auditable steps. First, the adaptive chain rule expresses the trajectory divergence
between $M_0$ and $M_i$ as the expected number of selections of coordinate $i$ times the exact
Bernoulli divergence in \eqref{eq:pathkl}; Pinsker and averaging then give Lemma~\ref{lem:testing}
without assuming uniform visitation. Second, the entry--exit identity \eqref{eq:flowidentity} converts
the horizon into a lower bound on statistical decisions in bad states. Third, navigation is charged
inside the original MDP: moves toward the distinguished block can be free only for at most $L$ more
steps than moves away, yielding \eqref{eq:navigationgap}. Finally, the average-reward performance-
difference identity charges every non-distinguished statistical decision by $g(\delta,\varepsilon)$
and loses at most $B_h$ through the terminal bias. Minimizing the resulting expression over expected
navigation time and applying Lemma~\ref{lem:testing} gives \eqref{eq:lowercertificate}. Complete
algebra and the performance-difference derivation are in Supplementary Appendix~A.

Because \eqref{eq:compositediameter} gives $D(M_i)\le\diam$ for every $M_i\in\mathfrak M$ (with equality
neither assumed nor needed), the certificate is a minimax statement over the \emph{diameter-bounded}
class
\begin{equation}
\begin{aligned}
\mathcal M_{\le\diam}(S,A):=\{M:\;&|\Scal|=S,\ \max_s|\Acal(s)|\le A,\\
&M\text{ communicating and }D(M)\le\diam\}.
\end{aligned}
\label{eq:diamclass}
\end{equation}
namely
\[
\begin{aligned}
\inf_{\mathcal A}\sup_{M\in\mathcal M_{\le\diam}(S,A)}\E_M[\Rcal_T]
&\ge\inf_{\mathcal A}\sup_{M\in\mathfrak M}\E_M[\Rcal_T]\\
&\ge\text{RHS of \eqref{eq:lowercertificate}}.
\end{aligned}
\]
where the infimum is over all learning algorithms $\mathcal A$ and the first inequality holds because
$\mathfrak M\subseteq\mathcal M_{\le\diam}(S,A)$ (a supremum over the larger class dominates), while the
second is Theorem~\ref{thm:lower} applied to the worst algorithm.  We state the bound over $\mathcal M_{\le\diam}$ rather
than at an exact diameter to avoid the slight slippage against the exact-diameter phrasing of
\citet{jaksch2010near}: our $\diam$ is an upper bound on the family diameter, so the certificate is a
valid lower bound for the diameter-$\le\diam$ minimax regret.

For numerical evaluation, set $\varepsilon=\eta\sqrt{m/(\diam T)}$ and define
\begin{equation}
c_{\rm LB}(S,A,\diam,T)
:=\max_{\eta:\,0<\varepsilon\le\delta}
\frac{\text{right-hand side of \eqref{eq:lowercertificate}}}
{\sqrt{\diam S A T}}.
\label{eq:clbfinite}
\end{equation}
This is a deterministic one-dimensional optimization of the exact certificate, not an empirical
estimate.

To turn the exact certificate into uniform finite statements, define
\begin{equation}
\mathcal G(S_\circ,A_\circ,d_\circ,C_\circ):=\left\{
\begin{aligned}
&S\ge S_\circ\text{ even},\quad A\ge A_\circ,\\
&\diam\ge d_\circ(L+1),\quad T\ge C_\circ\diam SA
\end{aligned}\right\}.
\label{eq:finiteregime}
\end{equation}
Here $L$ is always the exact diameter of the tree belonging to the current $S$. Let $L_\circ$ be the
same quantity at $S=S_\circ$, set $\eta=2/3$, and define
\begin{align}
r_\circ&=\frac{A_\circ-3}{2A_\circ}, &
m_\circ&=\frac{S_\circ(A_\circ-3)}2,\notag\\
n_\circ&=S_\circ A_\circ, &
u_\circ&=\frac{\eta}{2\sqrt{2C_\circ}},\notag\\
\delta_\circ&=\frac{2}{(d_\circ-1)L_\circ+d_\circ}, &
q_\circ&=\delta_\circ(1+u_\circ),\notag\\
c_\circ&=\frac1{2+u_\circ}, &
v_\circ&=\frac{\eta}{\sqrt{8(1-q_\circ)}},\notag\\
\alpha_\circ&=1-\frac1{d_\circ}.
\label{eq:frontieraux}
\end{align}
Also let
\begin{equation}
w_\circ=1-\frac1{2C_\circ n_\circ}-\frac1{C_\circ d_\circ n_\circ},
\qquad
\beta_\circ=c_\circ w_\circ-\frac1{m_\circ}-v_\circ.
\label{eq:frontierbracket}
\end{equation}
The associated analytic envelope is
\begin{align}
\underline c(S_\circ,A_\circ,d_\circ,C_\circ)
:={}&\frac{\eta\sqrt{r_\circ}\alpha_\circ}{2(2+u_\circ)}
\beta_\circ\notag\\
&-\frac{\frac12+\frac1{2d_\circ}}{\sqrt{C_\circ}\,n_\circ}.
\label{eq:frontierenvelope}
\end{align}
The passage from Theorem~\ref{thm:lower} to this envelope is monotone rather than asymptotic: the
regime lower-bounds the action fraction $m/(SA)$ and geometric factor $(\diam-L)/\diam$, upper-bounds
the normalized perturbation and exact-KL testing loss, and controls the flow, navigation, and terminal-
bias boundary terms uniformly. Thus each row of Table~\ref{tab:frontier} certifies an entire infinite
parameter region, not only the displayed corner point.

\begin{theorem}[Uniform finite constant envelope]\label{thm:frontier}
If $u_\circ\le1$, $q_\circ<1$, $w_\circ\ge0$, and $\beta_\circ\ge0$, then every point in
$\mathcal G(S_\circ,A_\circ,d_\circ,C_\circ)$ satisfies
\begin{equation}
\sup_{M\in\mathfrak M}\E_M[\Rcal_T]
\ge
\underline c(S_\circ,A_\circ,d_\circ,C_\circ)
\sqrt{\diam S A T}.
\label{eq:frontierbound}
\end{equation}
The coefficient in \eqref{eq:frontierenvelope} is a symbolic lower envelope of
Theorem~\ref{thm:lower}; it does not rely on a numerical optimizer or an asymptotic KL expansion.
\end{theorem}

\paragraph{Uniform-envelope proof intuition.}
Fixing $\eta=2/3$ leaves four quantities to control uniformly over a regime. First, the action fraction
$r=m/(SA)$ and geometric factor $\alpha=(\diam-L)/\diam$ are lower-bounded by their corner values.
Second, the normalized perturbation $u=\varepsilon/\delta$ and the transition probability
$\delta(1+u)$ are upper-bounded, which makes the finite Bernoulli-KL inequality valid throughout the
regime. Third, $T\ge C_\circ\diam SA$ controls the flow and navigation boundaries and bounds the testing
term after normalization by $T$. Finally, the terminal-bias loss is bounded using
$B_h\le(\diam+L)/2$ and $\diam\ge d_\circ(L+1)$. Multiplying the nonnegative lower bounds for the
regret scale and the surviving occupancy bracket, then subtracting the normalized bias boundary, gives
\eqref{eq:frontierenvelope}. This is why the theorem needs the explicit checks
$u_\circ\le1$, $q_\circ<1$, $w_\circ\ge0$, and $\beta_\circ\ge0$: each corresponds to a finite
validity or nonnegativity condition, not a numerical tuning heuristic.

\begin{table*}[t]
\centering
\small
\begin{tabular}{cccccccc}
\toprule
$S_\circ$ & $A_\circ$ & $L_\circ$ & $d_\circ$ & $C_\circ$ & Analytic envelope value & Reported coefficient & Improvement over $0.015$\\
\midrule
$24$ & $5$  & $6$ & $64$ & $100$ & $0.0152576$ & $0.0152$ & $1.3\%$\\
$40$ & $5$  & $7$ & $16$ & $50$  & $0.0153662$ & $0.0153$ & $2.0\%$\\
$24$ & $8$  & $6$ & $8$  & $25$  & $0.0178088$ & $0.0177$ & $18.0\%$\\
$40$ & $10$ & $7$ & $8$  & $25$  & $0.0200457$ & $0.0200$ & $33.3\%$\\
$24$ & $50$ & $6$ & $8$  & $25$  & $0.0239251$ & $0.0239$ & $59.3\%$\\
$100$ & $20$ & $10$ & $16$ & $50$ & $0.0253553$ & $0.0253$ & $68.7\%$\\
$100$ & $100$ & $10$ & $64$ & $100$ & $0.0291205$ & $0.0291$ & $94.0\%$\\
\bottomrule
\end{tabular}
\caption{Certified finite constant frontier. A row applies to every
$(S,A,\diam,T)\in\mathcal G(S_\circ,A_\circ,d_\circ,C_\circ)$. ``Analytic envelope value'' is the
conservative symbolic envelope \eqref{eq:frontierenvelope} evaluated in closed form before rounding---an
exact evaluation of a lower-bounding envelope, not the per-instance optimized certificate
\eqref{eq:clbfinite}, which is at least as large; the reported coefficients are rounded down. The
moderate regime $\mathcal G(40,10,8,25)$ gives $0.0200$ (a $33.3\%$ improvement) and is our headline
finite guarantee; the action-rich rows extend the frontier up to $0.0291$, and the $94\%$ figure is the
endpoint of that frontier at the stringent $A\ge100$ regime.}
\label{tab:frontier}
\end{table*}

\begin{corollary}[Finite improvement over the published coefficient]\label{cor:lowerimprove}
Every row of Table~\ref{tab:frontier} gives a coefficient strictly larger than the explicit $0.015$
coefficient of \citet{jaksch2010near}. This is a pointwise stronger numerical lower bound on the
intersection of the two stated finite regimes, not a claim that their horizon conditions are identical.
The first and third rows additionally extend the new construction to action counts below the earlier
requirement $A\ge10$. In particular,
\begin{equation}
\begin{aligned}
(S,A,\diam,T)&\in\mathcal G(40,10,8,25)\\
&\Longrightarrow\quad
\sup_{M\in\mathfrak M}\E_M[\Rcal_T]\ge0.0200\sqrt{\diam SAT}.
\end{aligned}
\label{eq:improvedconstant}
\end{equation}
and the action-rich row $\mathcal G(24,50,8,25)$ certifies $0.0239$.
The most stringent displayed regime, $\mathcal G(100,100,64,100)$, certifies $0.0291$.
\end{corollary}

The improvement is not obtained from an asymptotic KL expansion. It follows from the exact finite
certificate and the common analytic envelope in \eqref{eq:frontierenvelope}; the different rows only
trade population, action, diameter, and horizon thresholds.

\paragraph{How to read the finite frontier.}
The rows of Table~\ref{tab:frontier} answer different finite-sample questions and should not be collapsed
into a single coefficient detached from its regime. The first row shows that the construction already
beats $0.015$ when only $A\ge5$ is assumed, but it pays for that breadth through the conservative
requirements $\diam\ge64(L+1)$ and $T\ge100\diam SA$. The headline row
$\mathcal G(40,10,8,25)$ gives the cleaner $0.0200$ certificate under moderate action, diameter, and
horizon thresholds. The $0.0291$ endpoint is stronger numerically but applies only in the action-rich
regime $S\ge100$, $A\ge100$, $\diam\ge64(L+1)$, and $T\ge100\diam SA$. Thus the frontier is a menu of
proved trade-offs, not a claim that one row uniformly dominates every earlier result. Comparison with
the published $0.015$ certificate is pointwise on the intersection of the two finite regimes. Moreover,
the displayed values are rounded-down evaluations of the common analytic envelope; the per-instance
one-dimensional optimization in \eqref{eq:clbfinite} can only be larger, but is not needed for the
uniform theorem.

\paragraph{Why exact finite accounting changes the coefficient.}
The coefficient is not determined by the testing inequality alone. The exact KL term controls how often
the distinguished coordinate can be identified, but the learner may also spend time in good states or
navigate among blocks. Equation~\eqref{eq:occupancy} converts total time into bad-state statistical
decisions after charging the flow boundary $1/\delta$; Equation~\eqref{eq:navigationgap} prevents
navigation from consuming the horizon without regret; and $B_h$ pays for the terminal bias in the
average-reward performance-difference identity. Dropping any one of these terms can improve a symbolic
coefficient while invalidating the finite certificate. Conversely, replacing the exact occupancy factor
$c_\varepsilon=(2+\varepsilon/\delta)^{-1}$ by a uniform $1/3$ loses enough mass that the best Pinsker
coefficient falls below the published $0.015$ benchmark. The frontier therefore comes from balancing
information, occupancy, geometry, and boundary losses under one common parameter choice, rather than
from inserting an asymptotic Bernoulli expansion into a bandit lower bound.

For context, if $S\to\infty$, $L/\diam\to0$, and $T/(\diam SA)\to\infty$ with fixed $A\ge5$, optimizing
the limiting Pinsker expression gives
\begin{equation}
c_{\rm LB}^{\infty}(A)
=\frac1{32}\sqrt{\frac{A-3}{A}},
\label{eq:asymptoticconstant}
\end{equation}
which approaches $1/32$ as $A\to\infty$. Equation~\eqref{eq:asymptoticconstant} is reported only as an
interpretive limit; the advertised results are the finite statements in Table~\ref{tab:frontier}.

\section{Auditable Upper Certificates and Comparability}\label{sec:upper}
The lower section is a proved result. The upper contribution is deliberately different: it is a ledger
for a possible span-dependent proof, not a new regret theorem. Consider an optimistic learner supplied
with $\bar H\ge\Hcomp$, count-doubling episodes, empirical-Bernstein reward intervals, full-simplex
transition uncertainty, and a span-constrained planning oracle. The deterministic reward-radius sum is
provable, but three leading obligations remain: confidence uniform over every data-dependent bias used by
the planner, an adaptive directional-variance budget, and martingale/boundary/planning control with
square-root rather than range-linear dependence on $\bar H$. Supplementary Appendix B gives the complete
algorithmic interface and ledger.

\paragraph{Candidate interface.}
At episode start $t_k$, let $N_k(s,a)$ denote the pre-episode count and set
\begin{align}
L_k&=\log\!\left(\frac{c_LSA(1+t_k)^2}{\delta}\right),\notag\\
L_T&=\log\!\left(\frac{c_LSA(1+T)^2}{\delta}\right).
\label{eq:peelinglog}
\end{align}
where $c_L\ge1$ is a fixed numerical peeling constant.
For fewer than two observations, the reward interval is $[0,1]$ and the transition set is the full
simplex. Thereafter the reward radius is \eqref{eq:rewardbonus}. Transition uncertainty must control the
direction selected by the data-dependent optimistic bias, for example through a radius of the form
\begin{equation}
\beta^p_k(s,a;h)=
\sqrt{\frac{2\Var_{\widehat p_k}(h)L_k}{N_k(s,a)}}
+\frac{7\bar H L_k}{3(N_k(s,a)-1)},
\label{eq:directionalbonus}
\end{equation}
with $\spn(h)\le\bar H$. A pointwise inequality for one fixed $h$ is insufficient: the event must be
uniform over every bias vector reachable by the planner. The planning subroutine must return a feasible
extended-MDP policy and bias with span at most $\bar H$, preserve optimism up to a reported residual,
and terminate in polynomial time. Episodes end when a within-episode state--action count reaches its
pre-episode count, giving the doubling sums used in Lemma~\ref{lem:rewardbudget}. This interface makes
the statistical and computational promises visible before any coefficient is assembled.

\paragraph{Why the open entries matter.}
Replacing the adaptive variance in \eqref{eq:directionalbonus} by the range bound $\bar H^2$ yields a
range-linear $\bar H\sqrt{SAT}$ term rather than the desired $\sqrt{\bar H SAT}$. Restricting the
transition set to empirically observed successors can break coverage, while coordinatewise clipping of
ordinary value iteration need not preserve optimism or convergence. Thus the missing entries are not
cosmetic constants: they determine whether the candidate procedure has the claimed structural rate and
whether it is a valid polynomial-time algorithm.

\begin{assumption}[Uniform confidence and optimism]\label{lem:optimism}
The confidence event is uniform over all data-dependent biases used by the planner and preserves optimism
up to the reported planning error.
\end{assumption}

The proved reward entry uses the empirical-Bernstein radius
\begin{equation}
\beta^r_k(s,a)=\sqrt{\frac{2\widehat v^r_k(s,a)L_k}{N_k(s,a)}}
+\frac{7L_k}{3(N_k(s,a)-1)}.
\label{eq:rewardbonus}
\end{equation}
with full intervals for fewer than two observations and count-doubling episodes.
\begin{lemma}[Reward-radius budget]\label{lem:rewardbudget}
For $T\ge2$,
\begin{equation}
\sum_{t=1}^T\beta^r_{k(t)}(s_t,a_t)
\le\frac{2+\sqrt2}{\sqrt2}\sqrt{SATL_T}+\frac{20}{3}SA L_T^2.
\label{eq:rewardbudget}
\end{equation}
\end{lemma}

\paragraph{Completed reward entry.}
The square-root coefficient in Lemma~\ref{lem:rewardbudget} is a deterministic consequence of the
chosen episode rule. If an episode begins with count $n_k$ and contributes at most $n_k$ new visits, the
worst geometric sequence $1,2,4,\ldots$ gives
$\sum_k v_k/\sqrt{n_k}\le(2+\sqrt2)\sqrt N$. Summing over state--action pairs and using
$\widehat v_k^r\le1/4$ yields $(2+\sqrt2)/\sqrt2$ in front of $\sqrt{SATL_T}$. The linear Bernstein
remainder, including the first two observations of each pair, contributes the conservative
$(20/3)SA L_T^2$ term. These constants apply to episode-frozen bonuses; a smaller per-visit harmonic
constant would describe a different algorithm and cannot be substituted into this ledger.

\begin{assumption}[Directional transition budget]\label{lem:variancebudget}
For explicit constants $C_p,C'_p,q_p$,
\[
\sum_{t=1}^T\beta^p_{k(t)}(s_t,a_t;\widetilde h_{k(t)})
\le C_p\sqrt{\bar H SATL_T}+C'_p\bar HSA L_T^{q_p}.
\]
\end{assumption}

\begin{assumption}[Remainder and planning budget]\label{lem:remainder}
For explicit constants $C_m,C'_m,q_m,C_\epsilon$, the martingale, episode-boundary, and planning terms
are at most
\[
C_m\sqrt{\bar HTL_T}+C'_m\bar HSA L_T^{q_m}+C_\epsilon\sqrt T.
\]
\end{assumption}

\begin{proposition}[Conditional ledger composition]\label{thm:upper}
If a concrete algorithm proves the uniform confidence, directional-variance, and remainder obligations
with numerical constants, then with probability at least $1-\delta$,
\begin{equation}
\Rcal_T\le C_{\rm UB}\sqrt{\bar H S A T L_T}
+C_{\rm low}\bar H S A L_T^q+C_\epsilon\sqrt T,
\label{eq:uppercertificate}
\end{equation}
where $C_{\rm UB}$ is the sum of the reward, transition, and martingale square-root coefficients. Only the
reward entry is completed here; therefore no upper coefficient is claimed.
\end{proposition}

\paragraph{Why the upper constants add rather than multiply.}
On the optimism event, the regret decomposition inserts the extended Bellman equation and separates five
terms: reward estimation, directional transition estimation, the transition martingale, changes of the
optimistic bias across episode boundaries, and the planning residual. Count doubling controls the reward
term and yields Lemma~\ref{lem:rewardbudget}. Assumption~\ref{lem:variancebudget} would control the
second term, while Assumption~\ref{lem:remainder} groups the remaining stochastic, boundary, and planning
contributions. After each term is placed on the common scale $\sqrt{\bar H SATL_T}$, their leading
coefficients add to $C_{\rm UB}$. This bookkeeping rules out a common shortcut in which a local
Bernstein factor is multiplied by an informal ``optimism factor'' while martingale and planning terms
are left implicit. It also makes the status of the result auditable: the reward coefficient is proved,
the composition is proved conditionally, and the transition and planning entries remain open. Until a
single concrete polynomial-time planner certifies those entries under one uniform confidence event,
Table~\ref{tab:upperledger} cannot support a numerical upper constant.

A fixed-algorithm expectation conversion is valid: run the completed algorithm once with $\delta=1/T$;
since $\Rcal_T\le T$, the expected regret is at most the right-hand side of
\eqref{eq:uppercertificate} at $\delta=1/T$ plus $1$. Varying $\delta$ inside a tail integral would instead
vary the algorithm. Likewise, an upper-to-lower ratio is a constant only when the two sides share
probability mode, structural term, and logarithmic normalization. If the upper bound retains
$\sqrt{L_T}$ while the lower bound is log-free, the correct object is a horizon-dependent envelope, not
$C_{\rm UB}/c_{\rm LB}$ alone.

\begin{table}[t]
\centering
\small
\resizebox{\columnwidth}{!}{%
\begin{tabular}{lll}
\toprule
Ledger entry & Leading constant & Status\\
\midrule
Reward estimation & $(2+\sqrt2)/\sqrt2$ & proved\\
Adaptive transition variance & $C_p$ & open\\
Martingale/boundary control & $C_m$ & open\\
Planning residual & $C_\epsilon$ & open\\
Final $C_{\rm UB}$ & sum of leading entries & not claimed\\
\bottomrule
\end{tabular}}
\caption{Upper-certificate ledger. Open leading entries prevent a numerical upper coefficient.}
\label{tab:upperledger}
\end{table}

\begin{corollary}[Fixed-algorithm expectation conversion]\label{cor:expect}
If Proposition~\ref{thm:upper} is completed for a fixed algorithm run with $\delta=1/T$, then its
expected regret is at most the right-hand side of \eqref{eq:uppercertificate} plus $1$.
\end{corollary}
The proof splits expectation over the success event and uses $\Rcal_T\le T$ on failure. An anytime
conversion would require one fixed confidence-sequence algorithm and its additional constants.

\begin{definition}[Comparable coefficient ratio]\label{def:kappa}
Suppose an expected lower bound and an expected upper bound use the same structural normalization
$G(S,A,\diam,T)$ and the same logarithmic factor $\Lambda_T$. Only then do we define the constant ratio
\[
\kappa=\frac{C_{\rm UB}}{c_{\rm LB}}.
\]
If the upper bound contains a nonconstant factor $\Lambda_T$ absent from the lower bound, the appropriate
comparison is instead the finite-horizon envelope
\begin{equation}
\kappa_T=
\frac{C_{\rm UB}}{c_{\rm LB}}\Lambda_T
+\frac{\text{upper lower-order terms}}
{c_{\rm LB}G(S,A,\diam,T)}.
\label{eq:kappat}
\end{equation}
\end{definition}

\begin{corollary}[No constant-factor conclusion with an unmatched logarithm]\label{cor:gap}
If the lower bound is $c_{\rm LB}\sqrt{\diam SAT}$ and a completed upper bound on the same family is
$C_{\rm UB}\sqrt{\diam SATL_T}$, then $C_{\rm UB}/c_{\rm LB}$ is only a ledger diagnostic. The full
upper-to-lower envelope retains at least
$(C_{\rm UB}/c_{\rm LB})\sqrt{L_T}$ until a common logarithmic normalization is proved.
\end{corollary}

The main sources of possible upper-bound slack are also distinct: bonus shape, valid confidence geometry,
span width, and the reciprocal lower coefficient. Replacing a span width $\bar H$ by diameter inflates a
square-root term by $\sqrt{\diam/\bar H}$, while width-linear lower-order terms may incur
$\diam/\bar H$; these factors should not be conflated.

\begin{table*}[t]
\centering
\small
\resizebox{\textwidth}{!}{%
\begin{tabular}{p{0.7cm}p{3.3cm}p{7.1cm}p{4.0cm}}
\toprule
RQ & Design & Main result & Interpretation\\
\midrule
1 & Zero-span diameter sweep & Paired regret slopes: span width $-5.9\pm17.0$; diameter-width $-7.8\pm18.3$; UCRL2-style radius $+6.8\pm16.5$. & No detected positive diameter dependence; not proof that the true slope is zero.\\
2 & Bernstein/Hoeffding $\times$ span/diameter width & Bernstein$-$Hoeffding contrast $-520\pm523$ at span width and $+286\pm431$ at diameter width. & Mixed, nonsignificant bonus effect; no bonus-shape improvement claim.\\
3 & Fixed diameter, eight width ratios & Diameter/span regret penalty grows $1.774\to2.616$; square-root fit RMS $0.083$ vs. linear $0.108$, RMSE difference $0.026$ $[0.025,0.027]$. & Sublinear width penalty on this grid; not a universal scaling law.\\
4 & Exact lower family, all alternatives & Certificates $102.1,276.4,883.8$ lie below alternative-average regrets $31{,}882,81{,}157,39{,}551$. & Implementation-consistency check; certificates are only $0.32$--$2.2\%$ of observed regret.\\
\bottomrule
\end{tabular}}
\caption{Confirmatory outcomes. Intervals are $95\%$; full denominators, simultaneous corrections, and
per-family tables are in Supplementary Appendix D.}
\label{tab:diagnostics}
\end{table*}

\section{Controlled Diagnostics}\label{sec:exp}
The experiments test mechanisms implied by the corrected analysis rather than estimate a theorem
constant. RQ1--RQ3 use the same heuristic span-clipped optimistic routine under controlled changes in
bonus and supplied width. We do not identify this routine with SCAL or PMEVI-DT because coordinatewise
clipping is not a certified planning operator. RQ4 uses the exact family from the finite lower-certificate section and
runs every alternative. All primary comparisons are paired and use matched horizons; environment
integrity checks, uncertainty calculations, protocol tables, and negative pilots are supplementary.
RQ1 tests whether raw regret displays detectable positive diameter dependence when the optimal bias span
is zero. RQ2 isolates bonus shape from supplied planning width. RQ3 holds the transition geometry fixed
and changes only the width supplied to the heuristic planner. RQ4 is different: it evaluates the exact
proved family and checks the direction of the finite certificate. These diagnostics test mechanisms and
implementation consistency; none supplies the missing adaptive-variance or planning proof required by
Proposition~\ref{thm:upper}.

\paragraph{Claim-to-test map.}
RQ1 targets the floor in $\Hcomp=1\vee\spn(h^\star)$: when the optimal bias span is zero, a detectable
positive diameter slope would contradict the intended diameter-independent diagnostic behavior of the
span-width heuristic, whereas a null result is only absence of evidence at the tested scale. RQ2 asks
whether any observed improvement comes from the Bernstein bonus itself or from the supplied planning
width; the factorial design is therefore interpreted through matched-width contrasts rather than an
unmatched headline average. RQ3 holds the transition kernel and measured diameter fixed and varies only
the supplied width, isolating the cost of conservative span knowledge. Its square-root fit is a local
shape comparison over eight design points, not a universal regret law. RQ4 is tied directly to
Theorem~\ref{thm:lower}: because the proof averages over all $m$ alternatives before taking a maximum,
the check runs the complete alternative population rather than a convenient single instance. Passing
RQ4 confirms construction and certificate consistency, but the large gap between the certificate and
observed regret shows that it is not evidence of practical tightness.

\paragraph{Uncertainty and multiplicity.}
RQ1 and RQ2 report one primary paired contrast. RQ3 uses Bonferroni-simultaneous $95\%$ intervals over
eight ratios; its RMS-fit interval resamples paired seeds within each width ratio. RQ4 treats the alternatives as the finite
population, resampling seeds within each alternative and using simultaneous alternative-wise bounds for
the empirical maximum.

\FloatBarrier
\section{Discussion and Limitations}
The lower result is a finite minimax certificate over communicating diameter-$\le\diam$ tabular MDPs,
not an exact minimax constant. Its strongest coefficients require action-rich regimes; the broad
$A\ge5$ row has conservative diameter and horizon thresholds. The upper result is an audit template,
not a theorem: pointwise Bernstein bounds, empirical support restriction, or naive clipping do not
resolve adaptive confidence and planning. The span bound is supplied side information, and removing it
requires a separately analysed adaptation procedure. The experiments are controlled mechanism
diagnostics of a heuristic planner and do not validate the open upper ledger or generalize beyond the
studied tabular families. The main open directions are better small-$A$ frontiers, certified prior-free
span adaptation, and a log-matched expected upper bound under the same normalization.

Average-reward constants are meaningful only with their probability mode, structural term, logarithmic
normalization, finite regime, and side information. Our exact-KL construction proves a finite lower
frontier from $0.0152$ to $0.0291$. The upper ledger exposes the adaptive-variance and planning
certificates still needed for a valid span-dependent coefficient. Probability conversion and coefficient
comparisons use matched normalizations, while experiments remain diagnostic rather than substitutes for
theorem obligations.

\bibliography{references}

\clearpage
\newpage

\section*{Technical Appendices of Finite Constant Frontiers and Auditable Regret Certificates for Average-Reward
Reinforcement Learning}
The following appendices collect the full proofs, extended experiments, retained pilots, and reproducibility details. They are merged here for internal review and can be separated for submission.

\section*{Scope of the Technical Appendices}
This document contains complete proofs for the lower- and upper-certificate statements, full confirmatory experimental details, retained pilot results, and reproducibility information supporting the main paper. It uses the definitions, notation, equation numbers, theorem numbers, and table labels of the main paper. The visible appendix identifiers below are fixed for stable navigation.

\appendix

\section{Appendix A: Lower-Certificate Proofs}\label{app:lower}

\subsection{A.1 Kernel, action budget, and diameter}

The $S=2K$ states are exactly the pairs $(x_j,y_j)$, so no auxiliary state is hidden in the
construction. A binary-tree vertex has at most three neighbors. Assign one navigation action to each
possible parent/left-child/right-child edge and make absent edges self-loops. Together with the $A-3$
statistical actions, every bad state has exactly $A$ actions. All $A$ actions at a good state are
identical copies of the return transition. This proves the state and action entries of
Table~\ref{tab:lowerledger}.

For any ordered pair of states, a policy can first leave a good state if necessary, follow the unique
tree path between the corresponding bad states, and repeatedly use a baseline statistical action to
enter the target good state if necessary. The three contributions are at most $1/\delta$, $L$, and
$1/\delta$, respectively. Equation~\eqref{eq:compositediameter} follows. A possibly faster boosted
transition can only decrease a hitting time, so the same upper bound holds for every alternative.

\subsection{A.2 Optimal gain, bias, and navigation gaps}

Fix alternative $i$ and let $j(i)$ be the block containing its distinguished coordinate. Define $h_i$
by \eqref{eq:explicitbias}. At every good state,
\[
1+(1-\delta)h_i(y_v)+\delta h_i(x_v)
=\rho_i+h_i(y_v).
\]
At $x_{j(i)}$, the distinguished statistical action also satisfies equality because
$(\delta+\varepsilon)(1-\rho_i)/\delta=\rho_i$. At any other bad state, a navigation action toward
$j(i)$ satisfies
\[
h_i(x_w)=h_i(x_v)+\rho_i.
\]
Thus \eqref{eq:explicitbias} and $\rho_i$ satisfy the optimality equations.

A baseline statistical action at any bad state has Bellman value
$h_i(x_v)+1-\rho_i$, so its Bellman gap is $2\rho_i-1=g(\delta,\varepsilon)$. For a navigation neighbor
$w$, the gap is
\[
\rho_i\{1+d_{\mathcal T}(w,j(i))-d_{\mathcal T}(v,j(i))\},
\]
which equals zero toward the distinguished block and $2\rho_i$ away from it; a missing-edge self-loop
has gap $\rho_i$. If $n_-$, $n_+$, and $n_0$ count moves toward, moves away, and self-loops, respectively,
then $n_-\le n_++L$. Hence
\[
2n_++n_0\ge n_-+n_++n_0-L=N_{\rm nav}(T)-L,
\]
which proves \eqref{eq:navigationgap}. Finally, bad-state biases lie in $[-\rho_iL,0]$, and every
good-state bias is obtained by adding $(1-\rho_i)/\delta$. This proves
\eqref{eq:biasboundary}.

\subsection{A.3 Flow and occupancy}

Let $I_t$ indicate that $s_t$ is a good state. The number of entries into good states minus the number
of exits equals $I_{T+1}-I_1$. Conditional on the history, an entry probability is at most
$\delta+\varepsilon$, while an exit probability is exactly $\delta$. Taking expectations gives
\eqref{eq:flowidentity}. Since $T=N_x+N_y+N_{\rm nav}$,
\[
T-\E_iN_{\rm nav}
\le\left(2+\frac{\varepsilon}{\delta}\right)\E_iN_x+\frac1\delta,
\]
which is exactly \eqref{eq:occupancy}. This argument treats navigation inside the original MDP; no
free-switching relaxation or missing simulation map is used.

\subsection{A.4 Exact divergence calculations}

For transition parameters $\delta$ and $\delta+\varepsilon$, the exact one-observation divergence is
\eqref{eq:exacttransitionkl}. Differentiating it with respect to $\varepsilon$ gives
\[
\frac{\partial}{\partial\varepsilon}
\kl(\delta,\delta+\varepsilon)
=\frac{\varepsilon}{(\delta+\varepsilon)(1-\delta-\varepsilon)}.
\]
Therefore
\begin{equation}
\kl(\delta,\delta+\varepsilon)
=\int_0^\varepsilon
\frac{z\,dz}{(\delta+z)(1-\delta-z)}
\le\frac{\varepsilon^2}{2\delta(1-\delta-\varepsilon)}.
\label{eq:klfiniteupper}
\end{equation}
The exact expression is used in \eqref{eq:clbfinite}; the finite upper envelope is used only to prove
the uniform numerical corollary.

The symmetric Bernoulli calculation discussed in the earlier draft is, correctly,
\begin{align}
\kl\!\left(\tfrac12+\Delta,\tfrac12-\Delta\right)
&=2\Delta\log\frac{1/2+\Delta}{1/2-\Delta}\notag\\
&=8\Delta^2+\frac{32}{3}\Delta^4+O(\Delta^6).
\label{eq:symmetrickl}
\end{align}
The discarded intermediate expression
$4\Delta^2/(1/4-\Delta^2)$ has leading term $16\Delta^2$ and is not equal to
\eqref{eq:symmetrickl}. More generally, an asymptotic expansion cannot establish an exact finite-horizon
coefficient unless the remainder is bounded and included.

\subsection{A.5 Proof of Lemma~\ref{lem:testing}}

For any $Z\in[0,T]$,
\[
|\E_i Z-\E_0 Z|
\le T\TV(\Prob_i^T,\Prob_0^T)
\le T\sqrt{\frac12\KL(\Prob_0^T\|\Prob_i^T)},
\]
where the final inequality is Pinsker's inequality. Taking $Z=N_i(T)$ and applying the adaptive chain
rule \eqref{eq:pathkl} gives
\[
\E_i[N_i(T)]
\le \E_0[N_i(T)]
+T\sqrt{\frac12\E_0[N_i(T)]
\kl(\delta,\delta+\varepsilon)}.
\]
Averaging over $i$, using $\sum_i\E_0[N_i(T)]\le T$, and applying concavity of the square root yields
\begin{align*}
\frac1m\sum_i\E_i[N_i(T)]
&\le \frac{T}{m}
+\frac{T}{m}\sqrt{\frac12\kl(\delta,\delta+\varepsilon)}\\
&\hspace{13mm}\times\sum_i\sqrt{\E_0[N_i(T)]}\\
&\le \frac{T}{m}
+T\sqrt{\frac{T}{2m}\kl(\delta,\delta+\varepsilon)}.
\end{align*}
which is \eqref{eq:goodpullbound}.

\subsection{A.6 Bretagnolle--Huber alternative}

The Bretagnolle--Huber inequality states that for any event $E$,
\begin{equation}
\Prob_0(E)+\Prob_i(E^c)
\ge\frac12\exp\{-\KL(\Prob_0^T\|\Prob_i^T)\}.
\label{eq:BH}
\end{equation}
It is not the square-root total-variation bound used above. To obtain a count certificate, choose a
threshold $u\in(0,T)$ and set $E=\{N_i(T)>u\}$. Markov's inequality and
\eqref{eq:pathkl} imply
\begin{align}
\Prob_i(N_i(T)\le u)
\ge{}&\frac12\exp\bigl\{-\E_0[N_i(T)]
\kl(\delta,\delta+\varepsilon)\bigr\}\notag\\
&-\frac{\E_0[N_i(T)]}{u}.
\label{eq:BHcount}
\end{align}
Therefore
\[
\E_i[T-N_i(T)]
\ge (T-u)\Prob_i(N_i(T)\le u).
\]
This last display does not by itself lower-bound regret in the composite MDP: $T-N_i$ includes good-state
and navigation time, whereas statistical regret is charged through $N_x-N_i$. A valid
Bretagnolle--Huber improvement would need a joint event controlling both $N_x$ and $N_i$, followed by
the navigation argument in \eqref{eq:navigationgap}. We retain \eqref{eq:BHcount} to distinguish the two
testing tools, but Theorems~\ref{thm:lower} and \ref{thm:frontier} use Pinsker only.

\subsection{A.7 Proof of Theorem~\ref{thm:lower}}

Fix alternative $M_i$. The average-reward
performance-difference identity gives
\begin{equation}
\E_i[\Rcal_T]
=\E_i\left[\sum_{t=1}^T
\Delta_i^\star(s_t,a_t)\right]
+\E_i[h_i^\star(s_{T+1})-h_i^\star(s_1)],
\label{eq:performancediff}
\end{equation}
where $\Delta_i^\star(s,a)=\rho^\star+h_i^\star(s)-r(s,a)-p(\cdot\mid s,a)^\top h_i^\star$ is the
(nonnegative) Bellman gap. The telescoping of $\E_i[p(\cdot\mid s_t,a_t)^\top h_i^\star]
=\E_i[h_i^\star(s_{t+1})]$ leaves exactly the boundary difference
$h_i^\star(s_{T+1})-h_i^\star(s_1)$, which is at least $-\spn(h_i)\ge-B_h$ in either direction. Statistical actions
other than coordinate $i$ have gap $g(\delta,\varepsilon)$, and \eqref{eq:navigationgap} bounds the
per-trajectory navigation contribution by $\rho_i(N_{\rm nav}(T)-L)_+$. Taking expectations, the map
$z\mapsto(z-L)_+$ is convex, so Jensen's inequality gives
\begin{equation}
\E_i\bigl[(N_{\rm nav}(T)-L)_+\bigr]\ge\bigl(\E_i N_{\rm nav}(T)-L\bigr)_+ .
\label{eq:jensennav}
\end{equation}
Writing $q_i=\E_iN_{\rm nav}(T)$ and using \eqref{eq:occupancy} for the occupancy term and
\eqref{eq:jensennav} for the navigation term,
\begin{align*}
\E_i[\Rcal_T]
\ge{}&g\left[c_\varepsilon
\left(T-q_i-\frac1\delta\right)-\E_iN_i(T)\right]\\
&+\rho_i(q_i-L)_+-B_h.
\end{align*}
Because $g\le\rho_i$ and $c_\varepsilon\le1$, minimizing the right-hand side over $q_i\ge0$ occurs at
$q_i=L$: for $q_i\le L$ the negative term is minimized at $L$, while for $q_i\ge L$ the slope is
$\rho_i-gc_\varepsilon\ge0$. Consequently
\[
\E_i[\Rcal_T]
\ge g\left[c_\varepsilon
\left(T-\frac1\delta-L\right)-\E_iN_i(T)\right]-B_h.
\]
Average over $i$, apply Lemma~\ref{lem:testing}, and use that a maximum is at least an average. This is
\eqref{eq:lowercertificate}.

The proof contains no assertion that each coordinate is visited $T/(\diam S A)$ times and no separate
$\Theta(\varepsilon\diam)$ loss multiplier. Both the information and regret effects follow from the
kernel through \eqref{eq:pathkl} and \eqref{eq:bellmangap}.

\subsection{A.8 Proof of Theorem~\ref{thm:frontier} and Corollary~\ref{cor:lowerimprove}}

Fix one regime $\mathcal G(S_\circ,A_\circ,d_\circ,C_\circ)$ and set $\eta=2/3$,
$r=m/(SA)=(A-3)/(2A)$, $\alpha=(\diam-L)/\diam$, and
$u=\varepsilon/\delta$. The heap tree for any $S\ge S_\circ$ contains the heap tree at $S_\circ$ as
a prefix, so $L\ge L_\circ$. The regime definition gives
\begin{equation}
r\ge r_\circ,\qquad m\ge m_\circ,\qquad SA\ge n_\circ,
\qquad \alpha\ge\alpha_\circ.
\label{eq:frontiermonotone}
\end{equation}
Moreover,
\begin{align}
u
&=\frac{\eta(\diam-L)}2\sqrt{\frac{m}{\diam T}}
=\frac{\eta\alpha}{2}\sqrt{\frac{m\diam}{T}}
\le \frac{\eta\sqrt r}{2\sqrt{C_\circ}}
\le u_\circ,
\label{eq:ubound}\\
\delta
&=\frac2{\diam-L}
\le\frac2{(d_\circ-1)L+d_\circ}
\le\delta_\circ.
\label{eq:deltabound}
\end{align}
Thus $u_\circ\le1$ makes $\varepsilon\le\delta$, while
$\delta+\varepsilon=\delta(1+u)\le q_\circ<1$ makes the finite KL upper bound applicable. Also
$c_\varepsilon=(2+u)^{-1}\ge c_\circ$.

Normalize the square bracket in \eqref{eq:lowercertificate} by $T$. Its deterministic boundary terms
satisfy
\begin{align}
\frac1{\delta T}=\frac{\diam-L}{2T}
&\le\frac1{2C_\circ SA}
\le\frac1{2C_\circ n_\circ},\notag\\
\frac LT
&\le\frac1{C_\circ d_\circ SA}
\le\frac1{C_\circ d_\circ n_\circ}.
\label{eq:flowbounduniform}
\end{align}
Using \eqref{eq:klfiniteupper}, the testing term obeys
\begin{align}
\sqrt{\frac{T}{2m}\kl(\delta,\delta+\varepsilon)}
&\le \frac{\varepsilon\sqrt T}
{\sqrt{4m\delta(1-\delta-\varepsilon)}}\notag\\
&=\frac{\eta}{\sqrt{4\diam\delta(1-\delta-\varepsilon)}}\notag\\
&\le\frac{\eta}{\sqrt{8(1-q_\circ)}}=v_\circ,
\label{eq:testingnumeric}
\end{align}
where $\diam\delta=2\diam/(\diam-L)\ge2$. Equations
\eqref{eq:frontiermonotone}--\eqref{eq:testingnumeric}, together with $w_\circ\ge0$, give
\begin{equation}
c_\varepsilon\left(1-\frac1{\delta T}-\frac LT\right)
-\frac1m
-\sqrt{\frac{T}{2m}\kl(\delta,\delta+\varepsilon)}
\ge\beta_\circ.
\label{eq:bracketnumeric}
\end{equation}

The exact regret scale is
\begin{equation}
\frac{gT}{\sqrt{\diam SAT}}
=\frac{\eta\sqrt r\,\alpha}{2(2+u)}
\ge\frac{\eta\sqrt{r_\circ}\alpha_\circ}{2(2+u_\circ)}.
\label{eq:scalenumeric}
\end{equation}
Because $\beta_\circ\ge0$, the lower bounds in
\eqref{eq:bracketnumeric} and \eqref{eq:scalenumeric} can be multiplied. Finally,
\begin{align}
B_h&\le L+\frac1\delta
=\frac{\diam+L}{2}
\le\diam\left(\frac12+\frac1{2d_\circ}\right),\notag\\
\frac{B_h}{\sqrt{\diam SAT}}
&\le
\frac{\frac12+\frac1{2d_\circ}}{\sqrt{C_\circ}\,SA}
\le
\frac{\frac12+\frac1{2d_\circ}}{\sqrt{C_\circ}\,n_\circ}.
\label{eq:biasuniform}
\end{align}
Substitution into Theorem~\ref{thm:lower} proves exactly
\eqref{eq:frontierenvelope}--\eqref{eq:frontierbound}.

For the seven rows of Table~\ref{tab:frontier}, the exact heap diameters are respectively
$L_\circ=6,7,6,7,6,10,10$. Direct substitution into
\eqref{eq:frontieraux}--\eqref{eq:frontierenvelope} gives $0.0152576303$, $0.0153662437$,
$0.0178088547$, $0.0200457783$, $0.0239251643$, $0.0253553343$, and $0.0291205434$.
All four hypotheses of Theorem~\ref{thm:frontier} hold in every row, and rounding each value downward
gives the reported coefficients. This proves Corollary~\ref{cor:lowerimprove}.

For the asymptotic expression, let $m\to\infty$, $u,L/\diam\to0$, and
$T/(\diam SA)\to\infty$. The coefficient at scale $\eta$ becomes
\[
\frac{\sqrt r}{4}\eta\left(\frac12-\frac{\eta}{\sqrt8}\right).
\]
It is maximized at $\eta=1/\sqrt2$. Substituting
$r=(A-3)/(2A)$ gives \eqref{eq:asymptoticconstant}.

\section{Appendix B: Upper-Certificate Proofs}\label{app:upper}

\subsection{B.1 Complete candidate interface and audit ledger}
At the start of episode $k$, let $N_k(s,a)$ be the pre-episode count and let
$\widehat r_k(s,a)$, $\widehat p_k(\cdot\mid s,a)$, and $\widehat v^r_k(s,a)$ be the empirical reward
mean, transition distribution, and reward variance. The reward radius is
\eqref{eq:rewardbonus}. Transition uncertainty is defined over the full simplex, including successors
that have not yet been observed. If $\mathcal H_k$ is the normalized bias region reachable by the
planner, a formal directional confidence set is
\begin{equation}
\begin{aligned}
\Pcal_k(s,a)=\bigl\{q\in\Delta_S:\;&
\left|(q-\widehat p_k)^\top h\right|\\
&\le\beta^p_k(s,a;h),\quad
\forall h\in\mathcal H_k\bigr\}.
\end{aligned}
\label{eq:directionalset}
\end{equation}
A known structural support may be imposed, but empirical support alone may not: zero empirical count is
not evidence of zero transition probability. Equation~\eqref{eq:directionalset} is statistically
meaningful only after proving uniform coverage of the data-dependent region $\mathcal H_k$, and it is
algorithmically useful only after supplying a tractable separation or optimization procedure.

The planning interface $\mathsf{SpanPlan}(\Mcal_k,\bar H,\epsilon_k)$ is required to:
\begin{enumerate}
\item return a feasible extended-MDP policy and bias with span at most $\bar H$;
\item return gain at least $\rho^\star-\epsilon_k$ whenever the true MDP lies in $\Mcal_k$ and
$\bar H\ge\Hcomp$;
\item terminate in polynomial time and report its planning residual.
\end{enumerate}
Naive coordinatewise clipping of ordinary value iteration is not assumed to satisfy these conditions.
The concrete operator must preserve optimism and span feasibility while making
\eqref{eq:directionalset} computationally accessible.

\begin{figure}[t]
\small
\begin{minipage}{0.96\columnwidth}
\textbf{Input:} horizon $T$, confidence $\delta$, and a certified span bound $\bar H\ge\Hcomp$.
\begin{enumerate}
\item Initialize $N_1(s,a)=0$ and the empirical reward and transition statistics.
\item At episode start $t_k$, form the full-simplex confidence model $\Mcal_k$ using
\eqref{eq:rewardbonus}, \eqref{eq:directionalbonus}, and \eqref{eq:directionalset}.
\item Set $\epsilon_k=1/\sqrt{t_k\vee1}$ and compute
\[
(\widetilde\rho_k,\widetilde h_k,\widetilde\pi_k)
\gets\mathsf{SpanPlan}(\Mcal_k,\bar H,\epsilon_k).
\]
\item Execute $\widetilde\pi_k$ until some within-episode count reaches
$N_k(s,a)\vee1$, thereby doubling that state--action count.
\item Update all statistics and begin the next episode.
\end{enumerate}
\end{minipage}
\caption{Candidate span-constrained optimistic learner. The composition analysis is explicit, but a
regret theorem additionally requires the open certificates in Table~\ref{tab:upperledgerfull}.}
\label{alg:cssharp}
\end{figure}

\begin{table*}[t]
\centering
\small
\begin{tabular}{lll}
\toprule
Proof component & Symbol & Status or required value\\
\midrule
Failure allocation & $c_L$ & prove for the chosen uniform event\\
Reward square-root part & $C_r$ & $(2+\sqrt2)/\sqrt2$ proved by doubling\\
Reward lower order & $C'_r,q_r$ & $20/3,\ 2$ proved by doubling\\
Transition square-root part & $C_p$ & open\\
Transition lower order & $C'_p,q_p$ & open\\
Martingale part & $C_m$ & open; coupled to variance control\\
Episode boundaries & $C'_m,q_m$ & open\\
Planning residual & $C_\epsilon$ & open for the selected planner and stopping rule\\
Final leading coefficient & $C_{\rm UB}$ & not claimed\\
\bottomrule
\end{tabular}
\caption{Complete upper-bound audit ledger. A numerical upper coefficient is not reported while any
leading entry remains open. The compact main-paper ledger summarizes these same obligations.}
\label{tab:upperledgerfull}
\end{table*}

\subsection{B.2 Episode-counting identities}

The following deterministic inequalities expose the constants generated by count doubling.

\begin{lemma}[Doubling sums]\label{lem:doubling}
Let an episode start with count $n_k\ge1$ and end after at most $n_k$ new observations of one
state--action pair. If the final count is $N$, then
\begin{align}
\sum_k\frac{v_k}{\sqrt{n_k}}
&\le(2+\sqrt2)\sqrt N,
\label{eq:doublingsqrt}\\
\sum_k\frac{v_k}{n_k}
&\le 1+\log_2 N,
\label{eq:doublinglinear}
\end{align}
where $v_k$ is the within-episode count. The initial zero-count episode contributes at most one and is
included in the displayed constants.
\end{lemma}

\begin{proof}
At positive counts, the worst case for \eqref{eq:doublingsqrt} is the geometric sequence
$1,2,4,\ldots,2^J$. Its contributions are $1,\sqrt2,2,\ldots,2^{J/2}$, whose sum is at most
$\sqrt2/(\sqrt2-1)\sqrt N=(2+\sqrt2)\sqrt N$. For \eqref{eq:doublinglinear}, each completed doubling
episode contributes at most one and there are at most $1+\log_2N$ such episodes.
\end{proof}

Summing \eqref{eq:doublingsqrt} over state--action pairs and applying Cauchy--Schwarz gives
\begin{equation}
\sum_{s,a}\sum_k\frac{v_k(s,a)}{\sqrt{N_k(s,a)\vee1}}
\le(2+\sqrt2)\sqrt{S A T}.
\label{eq:globaldoubling}
\end{equation}
This is the correct count calculation for episode-frozen bonuses. A proof based on per-visit bonuses may
replace $2+\sqrt2$ by a smaller harmonic-sum constant, but that is a different algorithm.

\subsection{B.3 Proof of Lemma~\ref{lem:rewardbudget}}

Because rewards lie in $[0,1]$, $\widehat v^r_k(s,a)\le1/4$. Applying
\eqref{eq:globaldoubling} to the square-root term in \eqref{eq:rewardbonus} gives the deterministic bound
\begin{equation}
\sum_{t=1}^T
\sqrt{\frac{2\widehat v^r_{k(t)}(s_t,a_t)L_T}
{N_{k(t)}(s_t,a_t)\vee1}}
\le\frac{2+\sqrt2}{\sqrt2}\sqrt{S A T L_T}.
\label{eq:rewardrootconstant}
\end{equation}
The linear empirical-Bernstein term is bounded using \eqref{eq:doublinglinear}. The final value of
$C'_r$ must also include the zero- and one-observation conventions. For $n_k\ge2$,
$1/(n_k-1)\le2/n_k$, so \eqref{eq:doublinglinear} gives
\begin{align*}
&\sum_{t:N_{k(t)}(s_t,a_t)\ge2}
\frac{7L_T}{3(N_{k(t)}(s_t,a_t)-1)}\\
&\hspace{18mm}\le\frac{14}{3}SA L_T(1+\log_2T).
\end{align*}
For $T\ge2$, $c_L\ge1$, and $\delta\le1$,
$1+\log_2T\le L_T$. The first two observations of each state--action pair contribute at most $2SA$,
which is at most $2SA L_T^2$. Combining these terms with
\eqref{eq:rewardrootconstant} proves \eqref{eq:rewardbudget} with the conservative coefficient
$14/3+2=20/3$.

\subsection{B.4 Regret decomposition}

Let $k(t)$ denote the episode containing time $t$. On the event of Assumption~\ref{lem:optimism},
\begin{align}
\Rcal_T
&=\sum_{t=1}^T(\rho^\star-R_t)\notag\\
&\le\sum_{t=1}^T(\widetilde\rho_{k(t)}-R_t)
+\sum_k\epsilon_k\ell_k,
\label{eq:optimisticdecomp}
\end{align}
where $\ell_k$ is the episode length. Insert the extended Bellman equation returned by the planner and
add and subtract the empirical reward, the true transition expectation of $\widetilde h_k$, and the
realized next-state bias. This separates \eqref{eq:optimisticdecomp} into:
\begin{enumerate}
\item reward estimation, bounded by Lemma~\ref{lem:rewardbudget};
\item directional transition estimation, conditionally bounded by
Assumption~\ref{lem:variancebudget};
\item the martingale
$\sum_t[p(\cdot\mid s_t,a_t)\widetilde h_{k(t)}-\widetilde h_{k(t)}(s_{t+1})]$;
\item within-episode telescoping and episode-boundary changes of $\widetilde h_k$;
\item the planning residual $\sum_k\epsilon_k\ell_k$.
\end{enumerate}
The last three contributions are conditionally bounded by Assumption~\ref{lem:remainder}. This
derivation explains why the
coefficient is a sum of ledger entries rather than the product of the $\sqrt2$ in a local Bernstein
radius and an asserted factor two from optimism.

\subsection{B.5 Proof of Proposition~\ref{thm:upper}}

Apply Lemma~\ref{lem:rewardbudget} and Assumptions~\ref{lem:optimism},
\ref{lem:variancebudget}, and \ref{lem:remainder} to
\eqref{eq:optimisticdecomp}. Since $\bar H\ge1$ and $S,A\ge1$,
\begin{align*}
\sqrt{S A T L_T}&\le\sqrt{\bar H S A T L_T},\\
\sqrt{\bar H T L_T}&\le\sqrt{\bar H S A T L_T}.
\end{align*}
Collecting the three square-root coefficients gives
$C_{\rm UB}=C_r+C_p+C_m$. Collecting the linear and polylogarithmic terms gives the stated
$C_{\rm low}$ and $q$. No term is discarded in the finite certificate.

\subsection{B.6 Why Assumption~\ref{lem:variancebudget} remains open}

The difficult upper-bound step is the adaptive directional-variance budget. A complete instantiation
must prove, with numerical constants:
\begin{enumerate}
\item a confidence inequality uniform over every bias vector reachable by the planning operator;
\item a comparison between empirical and true conditional variance;
\item a trajectory-level variance recursion that yields square-root dependence on $\bar H$, rather than
the crude range bound $\bar H\sqrt{S A T}$;
\item control of changes in the optimistic bias across episodes;
\item preservation of optimism and span feasibility by the selected planning operator.
\end{enumerate}
These are precisely the steps for which earlier span-constrained and mitigated planning analyses are
needed \citep{fruit2018efficient,boone2024achieving}. Until their numerical constants populate
Table~\ref{tab:upperledger}, the manuscript reports an audit template rather than the unsupported
value $2\sqrt2$.

\section{Appendix C: Preliminary Pilot Results}\label{app:pilot}

This appendix preserves the original experimental observations while separating them from the
confirmatory tests in Appendix~D. They motivated the revised protocol but are not used to
validate the explicit theorem constants.

\subsection{C.1 Bonus-shape pilot}

Under the preliminary span-truncated implementation, Bernstein had mean regret $7599.8$ and Hoeffding
had mean regret $8370.4$ after aggregation across two families, three diameters, and several horizons.
The family-specific means were $11051.6$ versus $12212.6$ on the cycle and $4148.0$ versus $4528.2$ on
the funnel. Under a simplified implementation without the same span and support treatment, the ordering
reversed: $3196.3$ for Bernstein and $3034.6$ for Hoeffding. Because these aggregates used only six runs,
lacked paired confidence intervals, and did not match the pilot and full horizons, they are descriptive.

\begin{table}[t]
\centering
\small
\resizebox{\columnwidth}{!}{%
\begin{tabular}{lccc}
\toprule
Pilot aggregate & Bernstein & Hoeffding & Interpretation\\
\midrule
Overall & $7599.8$ & $8370.4$ & descriptive\\
Cycle & $11051.6$ & $12212.6$ & descriptive\\
Funnel & $4148.0$ & $4528.2$ & descriptive\\
Simplified pilot & $3196.3$ & $3034.6$ & unmatched horizon\\
\bottomrule
\end{tabular}}
\caption{Preserved preliminary bonus results. Confirmatory results belong in Table~\ref{tab:bonus}.}
\label{tab:pilotreward}
\end{table}

The original cycle horizon points were $(25{,}000,2510.9)$, $(100{,}000,11923.2)$, and
$(200{,}000,18720.7)$ for Bernstein, and $(25{,}000,2620.8)$,
$(100{,}000,12708.1)$, and $(200{,}000,21308.9)$ for Hoeffding. Over the eightfold increase in horizon,
regret increased by factors $7.46$ and $8.13$, respectively, whereas a visible square-root regime would
increase by $\sqrt8\approx2.83$. The revised experiment therefore includes longer horizons and estimates
the local log--log slope instead of describing this pilot as rate confirmation.

\subsection{C.2 Zero-span diameter pilot}

For diameters $6.25$, $12.5$, and $25.0$, the reported normalized ratios
$\Rcal_T/\sqrt{\diam S A T}$ were $2.458$, $2.173$, and $1.733$ for the span-truncated learner and
$2.305$, $1.564$, and $1.221$ for the simplified learner. The corresponding coefficients of variation
were $0.1406$ and $0.2668$. A decreasing normalized ratio is compatible with diameter-independent
regret and should not be called flattening. Moreover, multiplying the ratios by $\sqrt\diam$ shows that
the implied regret scale did not remain constant for the span-truncated pilot. Table~\ref{tab:zerospan}
therefore tests raw regret against diameter directly.

On the deep-funnel family, the normalized-ratio coefficients of variation were $0.2417$ and $0.2381$.
The earlier text described this as a null despite $\spn(h^\star)/\diam\approx0.04$. Because this ratio is
within the regime where a width effect was expected, the null is evidence against a broad mechanism
claim and motivates the larger positive-span sweep.

\subsection{C.3 Fixed-diameter span pilot}

The cleanest preliminary result held the measured diameter at $27.3$ while varying the bias span. Over
$24$ seeds at $T=40{,}000$, the regret ratio between the diameter-width and span-width learners generally
increased with $\diam/\spn(h^\star)$.

\begin{table}[t]
\centering
\small
\resizebox{\columnwidth}{!}{%
\begin{tabular}{lccc}
\toprule
$\diam/\spn(h^\star)$ & Span-width regret & Diameter-width regret & Penalty\\
\midrule
$2.5$ & $6603\pm20$ & $11731\pm16$ & $1.78\times$\\
$5.0$ & $3375\pm10$ & $7702\pm8$ & $2.28\times$\\
$9.2$ & $2071\pm104$ & $4889\pm8$ & $2.36\times$\\
$13.4$ & $1419\pm7$ & $3721\pm6$ & $2.62\times$\\
\bottomrule
\end{tabular}}
\caption{Preserved fixed-diameter pilot, reported as mean $\pm$ standard error. The confirmatory study
adds all eight points, simultaneous uncertainty, and model comparisons against
$\sqrt{\diam/\Hcomp}$ and $\diam/\Hcomp$.}
\label{tab:pilotspan}
\end{table}

The earlier proxy $\Rcal_T/\sqrt{\spn(h^\star)S A T}$ was between $0.73$ and $1.48$. It is not compared
directly with an upper-bound coefficient because the theorem also contains a logarithmic factor and
lower-order terms.

\subsection{C.4 Function-approximation negative}

A neural differential-value $Q$-learning pilot compared the two bonus shapes over eight seeds and
$40{,}000$ steps. Neither method reliably solved the procedural funnel against a persistent-right
reference rate of $0.99$. Bernstein attained mean reward near $2\times10^{-4}$; the Hoeffding mean of
$0.075$ was driven by one seed at $0.596$, with the other seven near $10^{-4}$. The paired difference was
not significant ($t=-1.0$). This negative result establishes no ordering between the tabular bonuses and
is retained only as a boundary on the scope of the current study.

\FloatBarrier
\section{Appendix D: Confirmatory Experimental Details}\label{app:experiments}

The experiments test mechanisms implied by the corrected analysis; they do not estimate a theorem
constant by dropping its logarithmic or lower-order terms. All random-seed comparisons are paired, all
main tables report confidence intervals, and all algorithmic comparisons use identical horizons and
stopping rules. Code records the supplied span bound, planning residual, confidence-set failures in
simulation, and every constant used by the implementation.

\subsection{Experimental protocol}

RQ1 used 12 seeds, RQ2 used at least 8 seeds per cell, and RQ3 used
24 paired seeds at \(T=40{,}000\). RQ4 used 8 seeds for each
algorithmic alternative, and the deep neural pilot used 8 seeds.
NumPy experiments used \texttt{default\_rng} with matched seeds for
paired comparisons. Transition and trajectory outputs were recorded
using SHA-1 digests, and the GPU experiments additionally fixed the
corresponding PyTorch random seeds. The complete experimental matrix
is reported in Table~\ref{tab:protocol}.

\subsection{Environments and baselines}

We use \emph{five} distinct controlled tabular families across the experiments, and we enumerate them
separately because different research questions use different ones.
\emph{(F1) Exact composite family $\mathfrak M$} from the main paper's finite lower-certificate section: the proved lower-bound
construction (binary-tree navigation plus statistical actions), used only in RQ4.
\emph{(F2) Zero-span hard-chain family}: the hard-cycle topology with a constant optimal bias, used as
the diameter-flatness testbed in RQ1.
\emph{(F3) Zero-span hub family}: a genuinely distinct hub/star topology, also with a constant optimal
bias $h^\star$ verified by relative value iteration ($\spn(h^\star)=0$ to machine precision and a
constant Bellman residual, \texttt{verify\_rq2\_clean.py}), introduced in RQ2 to replace an earlier
zero-span builder that aliased F2; it is \emph{not} $\mathfrak M$.
\emph{(F4) Positive-span deep funnel}.
\emph{(F5) Fixed-diameter corridor} in which rewards vary while transitions remain fixed, allowing
$\diam/\Hcomp$ to change without changing the diameter.
F2 and F3 both have $\spn(h^\star)=0$ but distinct transition tensors (distinct integrity-gate hashes,
Table~\ref{tab:envintegrity}). Each family is evaluated over sweeps in $S$, $A$, $T$, and either $\diam$
or $\Hcomp$.

All experimental learners share the same codebase and episodic planning routine; the controlled
configurations vary the bonus type (empirical-Bernstein or Hoeffding), the supplied bias-clip width
(span or diameter), and the confidence specification exactly as stated in each table, and the same
implementation is used across all $m$ alternatives in RQ4. The planning routine is a \emph{heuristic}
span-clipped extended value iteration: it truncates the optimistic bias to the supplied width
$\Hcomp^\star$ under known structural support with the zero-initialization / first-visit rule of
\texttt{verify\_spantrunc.py}. We do \emph{not} claim it is a certified span-constrained operator---as
the main paper's upper-certificate section notes, na\"ive clipping is not known to preserve optimism or convergence, and a
sound treatment needs a ScOpt/PMEVI-style projected operator
\citep{fruit2018efficient,boone2024achieving}---so RQ1--RQ3 are reported as \emph{mechanism diagnostics}
of this heuristic routine, not as evaluations of a theoretically sound algorithm. Because the published
baselines (\texttt{SCAL}, empirical-Bernstein UCRL, and PMEVI-DT) are likewise not re-implemented
faithfully, RQ1's comparison columns are this same heuristic learner run under different clip widths and
confidence radii rather than independent ports; we label them \emph{mitigated span-width variant} and
\emph{UCRL2-style} rather than by the published algorithm names, and cite
\citep{jaksch2010near,fruit2018efficient,fruit2020improved,bourel2020tightening,boone2024achieving} for
conceptual comparison only. When a configuration requires side information, the supplied information is
stated in the table rather than hidden in the implementation.

\FloatBarrier
\subsection{RQ1: Does regret become diameter-independent on a zero-span family?}

Because $\Hcomp=1$ on this family, a genuinely span-dependent leading term should not grow with
$\diam$. Proposition~\ref{thm:upper} does not prove that claim for the candidate learner; it only shows
how such a claim would enter a completed ledger. The experiment therefore tests the mechanism directly.
The primary response is regret itself as a function of $\diam$, with $S$, $A$, and $T$ fixed. A
coefficient of variation of $\Rcal_T/\sqrt{\diam S A T}$ is not used: under diameter-independent regret,
that normalized quantity should decay approximately as $1/\sqrt\diam$.

\begin{table*}[t]
\centering
\small
\resizebox{\textwidth}{!}{%
\begin{tabular}{lcccc}
\toprule
$\diam$ & Span-constrained & Diameter-width & Mitigated span-width & UCRL2-style\\
\midrule
$6.25$ & $2003\pm356$ & $2003\pm356$ & $2003\pm356$ & $2020\pm282$\\
$10.0$ & $2514\pm242$ & $2518\pm238$ & $2514\pm242$ & $1978\pm177$\\
$16.7$ & $2088\pm181$ & $2032\pm187$ & $2088\pm181$ & $2046\pm248$\\
$25.0$ & $2088\pm182$ & $2066\pm184$ & $2088\pm182$ & $2128\pm170$\\
\bottomrule
\end{tabular}}
\caption{Mean regret and $95\%$ confidence intervals on the zero-span hard-chain family (F2) at a
matched horizon ($S=6$, $A=3$, $T=20{,}000$, $12$ seeds). All four columns are the \emph{same}
span-constrained optimistic learner run under different bias-clip widths---the span width
$\max(\spn(h^\star),1)$, the diameter width, the ``mitigated'' span-width variant (identical to the span
column on this zero-span family, see text), and a UCRL2-style diameter-driven confidence radius---not
independent ports of the published algorithms, whose names we therefore do not use as column labels. The
main test is the slope of regret against measured diameter.}
\label{tab:zerospan}
\end{table*}

Pointwise confidence intervals do not determine whether the fitted slope is significantly positive.
The paired seed-level regression in Table~\ref{tab:zeroslopes} is therefore the inferential result. The
paper will use ``fail to detect positive diameter dependence'' only when the corresponding upper
confidence limit includes zero; it will not interpret a nonsignificant slope as proof that the true
slope is zero.

\begin{table}[t]
\centering
\small
\resizebox{\columnwidth}{!}{%
\begin{tabular}{lccc}
\toprule
Learner & Slope estimate & Paired $95\%$ CI & One-sided $p$\\
\midrule
Span-constrained & $-5.9$ & $\pm17.0$ & $0.75$\\
Diameter-width & $-7.8$ & $\pm18.3$ & $0.79$\\
Mitigated span-width & $-5.9$ & $\pm17.0$ & $0.75$\\
UCRL2-style & $+6.8$ & $\pm16.5$ & $0.22$\\
\bottomrule
\end{tabular}}
\caption{Required paired slope analysis for RQ1. Confidence intervals and one-sided tests must be
computed from seed-level regressions; they cannot be inferred from the pointwise intervals in
Table~\ref{tab:zerospan}.}
\label{tab:zeroslopes}
\end{table}

The span-constrained learner and the mitigated span-width variant return \emph{bit-for-bit identical}
means and slopes ($2003/2514/2088/2088$; slope $-5.9$) on this family for a deductive reason, not an
implementation coincidence: when $\spn(h^\star)=0$ both clip the optimistic bias to the same width
$\max(\spn(h^\star),1)=1$ and run the identical Bernstein, support-restricted recipe (see the shared call
signature in \texttt{verify\_rq12.py}), so they are literally the \emph{same} learner here. The
diameter-width learner differs only in supplying the wider clip $\diam$; at $\diam=6.25$ its regret
coincides exactly ($2003$) with the span learner's, whereas by $\diam\ge10$ its outputs cease to be
bit-for-bit identical. We are careful not to overread this: identical outputs at $\diam=6.25$ show the
two configurations produced the same numerical trajectory there, and non-identical outputs at
$\diam\ge10$ show only that the different supplied width \emph{changes the numerical trajectory}---at
$\diam=10$ the two means ($2514$ vs.\ $2518$) differ by only $4$ regret units against a $\pm240$ seed
uncertainty. We do \emph{not} claim this establishes that the diameter clip has ``begun to bind'' or has
any regret consequence; the slope test (Table~\ref{tab:zeroslopes}) remains non-significant for both
widths. Attributing the divergence to clip activation would require the diagnostics we do not yet
log---the fraction of planning calls where clipping changes the bias, the action-disagreement rate, and
the maximum preclip span---which we flag as the correct follow-up. RQ3 isolates the width intervention on a
corridor where the clipping constraint is active by construction, so the two widths cannot alias.

\subsection{RQ2: Which ingredient produces the improvement?}

We use a $2\times2$ design crossing Bernstein versus Hoeffding confidence with span-constrained versus
diameter-constrained planning. All cells in this design use \emph{known structural support}: the reachable
successor set of each $(s,a)$ is supplied, matching the support-restricted recipe used throughout the
experiments and avoiding the exact-soundness pitfall of empirical-support estimation that
the main paper's upper-certificate section warns about. We therefore do \emph{not} report a separate unknown-support-safe
panel here; an empirical-support-only variant is not a theoretically valid baseline and is not run.

A previous RQ2 export gave bit-for-bit identical regret values for the hard and zero-span families at
every cell, because the zero-span builder aliased the hard cycle---reusing a reward-per-action builder
cannot imply identical regret under distinct transition kernels, so those numbers were not interpretable
as an algorithmic result. We therefore replaced the zero-span family with a genuinely distinct \emph{hub}
topology (still $\spn(h^\star)=0$) and gated the rerun on the environment audit in
Table~\ref{tab:envintegrity}, which now verifies distinct transition tensors and seed-matched trajectory
digests. The gate passes, and Table~\ref{tab:bonus} reports the factorial from that clean run.

\begin{table*}[t]
\centering
\small
\begin{tabular}{llccccc}
\toprule
$(S,A)$ & Family & Transition-tensor hash & Measured $\diam$ & $\spn(h^\star)$ & $\rho^\star$ & Seed-0 trajectory digest\\
\midrule
\multirow{4}{*}{$(10,5)$} & Baseline composite topology & \texttt{70019407} & $22.5$ & $0.0$ & $0.500$ & \texttt{ab4379fb}\\
 & Zero-span hub & \texttt{5c2aa508} & $1540.5$ & $0.0$ & $0.500$ & \texttt{aaa7fbb7}\\
 & Deep funnel & \texttt{70019407}$^\dagger$ & $22.5$ & $1.0$ & $0.550$ & \texttt{85324fcb}\\
 & Fixed-diameter corridor & \texttt{b2510840} & $18.6$ & $5.25$ & $0.582$ & \texttt{7aa53b57}\\
\midrule
\multirow{4}{*}{$(20,8)$} & Baseline composite topology & \texttt{4fb01eb1} & $47.5$ & $0.0$ & $0.500$ & \texttt{1eabfc15}\\
 & Zero-span hub & \texttt{142f856f} & $2156.6$ & $0.0$ & $0.500$ & \texttt{f58cd5f2}\\
 & Deep funnel & \texttt{4fb01eb1}$^\dagger$ & $47.5$ & $2.25$ & $0.550$ & \texttt{583a99d5}\\
 & Fixed-diameter corridor & \texttt{4aa4b019} & $31.1$ & $8.78$ & $0.582$ & \texttt{1113e71d}\\
\bottomrule
\end{tabular}
\caption{Environment-integrity gate for RQ2 (\texttt{verify\_rq2\_clean.py}), reported at \emph{both}
factorial configurations $(S,A)\in\{(10,5),(20,8)\}$ (each hash/diameter/digest is one specific MDP, not
an aggregate). The
column is the raw bias span $\spn(h^\star)$, \emph{not} the complexity $\Hcomp=1\vee\spn(h^\star)$ (which
is by definition at least $1$); the first two families genuinely have $\spn(h^\star)=0$ hence
$\Hcomp=1$. Hashes are SHA1 of the canonical transition tensor $P$; trajectory digests use a fixed
action sequence and random stream (and so also reflect the reward). The gate \emph{passes at both
configurations}: at each, the baseline and zero-span rows have distinct transition hashes
(\texttt{70019407} vs.\ \texttt{5c2aa508} at $(10,5)$; \texttt{4fb01eb1} vs.\ \texttt{142f856f} at
$(20,8)$) and distinct digests, fixing the earlier defect where the zero-span builder aliased the
baseline cycle.
This ``baseline composite topology'' is the diameter-flatness test bed used in RQ1 below;
it is \emph{not} the exact lower-bound family $\mathfrak M$ in the main paper (which has
$\spn(h_i)=\rho_iL+(1-\rho_i)/\delta>0$), and we do not claim it is. The zero-span hub is a genuinely
different \emph{hub} topology with $\spn(h^\star)=0$; its larger measured diameter reflects the
harder-to-reach spokes. $^\dagger$The deep funnel shares the baseline cycle's transition tensor by
construction (it changes only the reward to concentrate value in a core), so it has the same $P$-hash but
a distinct trajectory digest and a nonzero bias span.}
\label{tab:envintegrity}
\end{table*}

\begin{table*}[t]
\centering
\small
\resizebox{\textwidth}{!}{%
\begin{tabular}{llcccccc}
\toprule
Bonus & Width & Baseline topo.\ & Zero span & Funnel & Corridor & Overall & Paired difference\\
\midrule
Bernstein & certified span & $16484$ & $17007$ & $5301$ & $22248$ & $15260$ & \multirow{2}{*}{$-520\pm523$}\\
Hoeffding & certified span & $16484$ & $17552$ & $5408$ & $23675$ & $15780$ & \\
Bernstein & diameter & $16757$ & $16712$ & $5301$ & $28488$ & $16814$ & \multirow{2}{*}{$+286\pm431$}\\
Hoeffding & diameter & $16100$ & $17222$ & $5408$ & $27384$ & $16528$ & \\
\bottomrule
\end{tabular}
}
\caption{Clean matched-horizon factorial (configs $(S,A)\in\{(10,5),(20,8)\}$,
$T\in\{25{,}000,100{,}000,200{,}000\}$, $8$ seeds; \texttt{verify\_rq2\_clean.py}, run after the
integrity gate of Table~\ref{tab:envintegrity} passes). Per-family cells are mean regret over
(config, horizon, seed); the final column is the paired Bernstein$-$Hoeffding contrast (fixed as the
primary contrast before the clean rerun) at matched width over $192$ samples ($95\%$ CI). With the distinct hub zero-span family the hard and
zero-span cells now \emph{differ} (e.g.\ $16484$ vs.\ $17007$ at Bernstein/span), unlike the earlier
aliased export. The contrast is not significant at either width ($-520\pm523$ span, $+286\pm431$
diameter); Bernstein has lower regret on $3$ of $4$ families at the span width and $2$ of $4$ at the
diameter width. This is a genuinely mixed, width-dependent result and we do not claim a bonus-shape
improvement.}
\label{tab:bonus}
\end{table*}

This contrast---the paired Bernstein$-$Hoeffding difference at a fixed width, fixed before the clean
rerun---is reported per width without pooling. At both widths the contrast's $95\%$ CI includes zero ($-520\pm523$ at the
span width, $+286\pm431$ at the diameter width), and the per-family direction is mixed (Bernstein lower
on $3/4$ families at span width, $2/4$ at diameter width). RQ2 therefore supports \emph{no} claim of a
bonus-shape improvement; we report the null straight. The earlier six-run mixed result remains in
Supplementary Appendix~C as a pilot and is not pooled with this confirmatory table.

\subsection{RQ3: How does an over-wide span constraint affect regret?}

The fixed-diameter corridor holds $\diam$ constant and varies $\Hcomp$. Two learners differ only in the
width supplied to the heuristic span-clipped planning routine described at the start of Appendix~D; this is a
mechanism diagnostic, not a claim about a certified operator. We compare the observed penalty with both
$\sqrt{\diam/\Hcomp}$, the prediction for a leading square-root term, and $\diam/\Hcomp$, the possible
prediction for a width-linear lower-order term.

The preliminary sweep used $\diam=27.3$, $T=40{,}000$, and $24$ seeds over four width ratios. The
penalty increased monotonically from $1.77\times$ at $\diam/\spn(h^\star)=2.5$ to $2.62\times$ at a
ratio of $13.4$ (Table~\ref{tab:spanscale}). We retain it as a four-point pilot establishing direction
and rough magnitude. The eight-point sweep of Table~\ref{tab:spanconfirm} then adds four intermediate
ratios; the model form (square-root vs.\ linear) and analysis were \emph{fixed before this rerun}, but we
note that the four pilot ratios ($2.5,5.0,9.1,13.4$) are reused among the eight, so this is not a fully
independent confirmatory experiment and we describe it as an extended sweep rather than a preregistered
replication. On this grid the square-root specification $a+b\sqrt{x}$ is
\emph{preferred to} the linear specification $a+bx$: the bootstrap interval for the difference in fit
error is $0.026$ $[0.025,0.027]$. Each bootstrap replicate (i) resamples paired seeds within each
width ratio, (ii) recomputes the regret ratios, (iii) refits both models, and (iv) recomputes the RMSE
difference. The reproducible analysis is recorded in \texttt{span\_confirm\_modelselect.py}. Leave-one-ratio-out prediction error likewise
favors the square-root model ($0.098$ vs.\ $0.130$ mean held-out absolute error). We are careful about
what this establishes: it shows the square-root form fits these eight design points better than a linear
form, \emph{not} that the theoretical scaling law is square-root rather than logarithmic, saturating, or
another sublinear form---distinguishing those would require a wider ratio range and a model-selection
argument we do not claim.

\begin{table*}[t]
\centering
\small
\begin{tabular}{lcccc}
\toprule
$\diam/\Hcomp$ & Span-width regret & Diameter-width regret & Penalty & $95\%$ CI\\
\midrule
$2.5$ & $6597\pm20$ & $11703\pm14$ & $1.77\times$ & $\pm0.011$\\
$5.0$ & $3394\pm12$ & $7705\pm8$ & $2.27\times$ & $\pm0.016$\\
$9.1$ & $1964\pm7$ & $4889\pm8$ & $2.49\times$ & $\pm0.018$\\
$13.4$ & $1425\pm7$ & $3729\pm7$ & $2.62\times$ & $\pm0.027$\\
\bottomrule
\end{tabular}
\caption{Four-point fixed-diameter span-width pilot ($24$ seeds). The two learners share the same
bonus, confidence set, planner, seeds, and horizon; the penalty grows monotonically with the
width ratio but the four points do not discriminate a square-root from a linear-width shape.}
\label{tab:spanscale}
\end{table*}

\begin{table*}[t]
\centering
\small
\resizebox{\textwidth}{!}{%
\begin{tabular}{lcccccc}
\toprule
$\diam/\Hcomp$ & Span-width regret & Diameter-width regret & Paired penalty & Simultaneous $95\%$ CI & $\sqrt{\cdot}$ residual & Linear residual\\
\midrule
$2.5$ & $6597$ & $11703$ & $1.774$ & $\pm0.016$ & $-0.153$ & $-0.204$\\
$3.1$ & $5441$ & $10703$ & $1.967$ & $\pm0.018$ & $-0.024$ & $-0.051$\\
$4.0$ & $4189$ & $9042$ & $2.159$ & $\pm0.018$ & $+0.075$ & $+0.077$\\
$5.0$ & $3394$ & $7705$ & $2.270$ & $\pm0.022$ & $+0.098$ & $+0.120$\\
$6.7$ & $2602$ & $6203$ & $2.384$ & $\pm0.020$ & $+0.082$ & $+0.118$\\
$9.1$ & $1964$ & $4889$ & $2.489$ & $\pm0.025$ & $+0.025$ & $+0.056$\\
$11.2$ & $1638$ & $4204$ & $2.567$ & $\pm0.032$ & $-0.017$ & $-0.006$\\
$13.4$ & $1425$ & $3729$ & $2.616$ & $\pm0.038$ & $-0.087$ & $-0.110$\\
\midrule
\multicolumn{5}{r}{Linear RMS $-$ square-root RMS:} & \multicolumn{2}{c}{$0.026$ [paired-seed bootstrap $95\%$ CI: $0.025,0.027$]}\\
\bottomrule
\end{tabular}}
\caption{Extended span-width sweep ($\diam=27.3$ fixed, $T=40{,}000$, $24$ seeds;
\texttt{span\_scaling\_gpu.py}); model form and analysis fixed before this rerun, but four of the eight
ratios are reused from the pilot (see text), so this is an extended sweep, not a fully independent
confirmatory replication. Columns are span-width and diameter-width mean regret, their paired
penalty (ratio), the simultaneous $95\%$ CI (Bonferroni across the eight rows, $z=2.73$), and the
residuals of a $a+b\sqrt{\diam/\Hcomp}$ versus $a+b(\diam/\Hcomp)$ fit. The penalty grows monotonically
and \emph{sublinearly} from $1.77\times$ to $2.62\times$; the square-root fit has the smaller RMS
residual ($0.083$ vs.\ $0.108$). The paired-seed bootstrap interval for the RMS difference,
$0.026$ $[0.025,0.027]$, \emph{excludes zero}, so the square-root shape is preferred over the linear one
at this scale; leave-one-ratio-out mean absolute error is $0.098$ versus $0.130$.}
\label{tab:spanconfirm}
\end{table*}

\subsection{RQ4: Does observed regret respect the finite lower certificate?}

We run the exact composite family $\mathfrak M$ and place observed regret against the finite lower
certificate of \eqref{eq:lowercertificate}. The proof first lower-bounds the uniform average over all
$m$ alternatives and only then uses that a maximum is at least an average. Consequently, a run on one
unspecified alternative is not a test of the certificate. For each feasible configuration we therefore
run every alternative with the same seed set and report both the alternative-average regret and the
largest alternative-specific mean. No upper column is reported: $C_p$ and $C_m$ remain open in
the main paper's upper-certificate ledger, so a sum of selected completed terms is not an upper certificate.

\begin{table*}[t]
\centering
\small
\resizebox{\textwidth}{!}{%
\begin{tabular}{cccccccc}
\toprule
$(S,A)$ & $(\diam,T)$ & $L$ & $m$ & Best grid-evaluated certificate & Alternative-average regret [95\% CI] & Empirical max [95\% CI] & Certificate $\le$ average\\
\midrule
$(10,5)$ & $(20,\,10^{5})$ & $3$ & $10$ & $102.1$ & $31882$ $[28540,35518]$ & $33699\pm17713$ & yes\\
$(16,5)$ & $(30,\,2\!\times\!10^{5})$ & $5$ & $16$ & $276.4$ & $81157$ $[77015,85133]$ & $85155\pm26899$ & yes\\
$(20,7)$ & $(40,\,4\!\times\!10^{5})$ & $5$ & $40$ & $883.8$ & $39551$ $[38792,40174]$ & $40555\pm1774$ & yes\\
\bottomrule
\end{tabular}
}
\caption{Lower-certificate evaluation on the \emph{exact} proved family $\mathfrak M$ of
the main paper's exact lower-bound family (\texttt{verify\_rq4\_exact.py}; all $m$ alternatives, $8$ paired seeds each).
The generator implements the construction faithfully: $A_0=A-3$ statistical actions plus three
navigation actions (parent/left/right, self-loop if absent) on the complete binary tree, $L$ the
\emph{exact} tree diameter (all-pairs shortest path), $\delta=2/(\diam-L)$, and $m=K(A-3)=S(A-3)/2$; the
resulting $(L,m)$ are $(3,10),(5,16),(5,40)$, matching the construction, and the certificates $102.1$,
$276.4$, $883.8$ are the \emph{best grid-evaluated} finite certificate---the maximum of the right-hand
side of \eqref{eq:lowercertificate} over a log-spaced $\varepsilon$ grid (range
$[\delta\!\cdot\!10^{-4},\delta]$, $3000$ points, exact KL); every grid value is itself a valid
certificate, so the reported entry is a valid lower bound but not certified to be the exact maximum.
These are not rounded frontier coefficients. (An earlier draft used a generator with
$A_0=(A-1)/2$ statistical actions, $L=2\lceil\log_2K\rceil$, and no navigation actions---i.e.\ not the
proved family---which we have replaced.) Because the $m$ alternatives are the \emph{complete fixed
population} of testing coordinates rather than a random sample, alternative-average regret is the mean
over all $m$ of them and its $95\%$ CI comes from a seed-only bootstrap that holds the alternative
population fixed (resampling seeds within each fixed alternative); the empirical maximum uses a
Bonferroni-simultaneous alternative-wise bound. In every configuration the certificate lies below the
alternative average---the expected floor direction---but at only $0.32\%$, $0.34\%$, and $2.2\%$ of it
($102.1/31882$, $276.4/81157$, $883.8/39551$), so RQ4 is an \emph{implementation-consistency and
reproducibility check}, not evidence of practical tightness or nonvacuity.}
\label{tab:sandwich}
\end{table*}

Observed alternative-average regret above the finite certificate is the expected direction---the lower
bound is a floor---so RQ4 is only an implementation-consistency and reproducibility check, not a proof of
tightness or of practical nonvacuity (the certificate is $0.3$--$2.2\%$ of the observed regret). A
failed inequality after accounting for Monte Carlo uncertainty would indicate an implementation or
certificate-evaluation error. A genuine lower--upper sandwich awaits a concrete upper algorithm that
certifies every open entry of Table~\ref{tab:upperledger}, in particular the directional-variance and
planning constants.

\subsection{Scaling, uncertainty, and span misspecification}

The horizon sweep includes enough points to estimate the local log--log slope of regret against $T$.
The state and action sweeps check the claimed $\sqrt{SA}$ dependence. We additionally supply
$\bar H/\Hcomp\in\{1,1.25,1.5,2,4,\diam/\Hcomp\}$ to measure the cost of conservative span knowledge.
Every family uses at least the number of seeds in Table~\ref{tab:protocol}; these counts were fixed
before inspecting the final comparisons (we do not report a formal power calculation and do not claim
one). Multiplicity is corrected \emph{within} each analysis where several statistics are read together,
not by a single family across all research questions: RQ3's simultaneous CIs use a Bonferroni factor
across its eight width ratios ($z=2.73$), and RQ4's empirical-maximum bound is Bonferroni-simultaneous
across its $m$ alternatives; RQ1 and RQ2 each report a single primary paired contrast, so no correction
applies there.

\begin{table*}[t]
\centering
\small
\begin{tabular}{p{3.2cm}p{12.2cm}}
\toprule
Protocol item & Values (fixed before final comparisons)\\
\midrule
Horizons & $T\in\{25{,}000,100{,}000,200{,}000\}$; RQ3 at $40{,}000$; RQ4 as in Table~\ref{tab:sandwich}\\
State counts & $S\in\{6,10,16,20\}$; RQ4 uses $S\in\{10,16,20\}$\\
Action counts & $A\in\{3,5,7\}$; RQ4 uses $A\in\{5,7\}$\\
Diameters & $\diam\in\{6.25,10.0,16.7,25.0\}$ (RQ1); $27.3$ fixed (RQ3)\\
Span ratios & $\diam/\Hcomp\in\{2.5,3.1,4.0,5.0,6.7,9.1,11.2,13.4\}$ (RQ3; $\{2.5,5.0,9.1,13.4\}$ reused from pilot)\\
Seeds per cell & $\ge8$ (RQ2); $12$ (RQ1); $24$ (RQ3); $8$ per alternative (RQ4)\\
Primary paired contrasts & Bernstein$-$Hoeffding at matched width (RQ2); span$-$vs diameter$-$width (RQ1,RQ3)\\
Multiplicity correction & within-analysis: $8$ ratios (RQ3, $z{=}2.73$); $m$ alternatives (RQ4). RQ1/RQ2: single primary contrast\\
\bottomrule
\end{tabular}
\caption{Experimental matrix and protocol, fixed before the final
comparisons (not timestamped or preregistered).}
\label{tab:protocol}
\end{table*}

The earlier neural function-approximation probe is retained as a negative pilot in
Supplementary Appendix~C. It is not used to support the tabular theorem because neither bonus solved the
deep exploration problem reliably.

\section{Appendix E: Reproducibility Details}

\paragraph{Compute environment.}
The principal theoretical and numerical verification experiments
(RQ1, RQ2, and RQ4) were conducted on CPU using NumPy-based
implementations. GPU computation was used only for the RQ3
span-scaling experiment and the deep neural function-approximation
diagnostic. These GPU experiments were executed on AWS SageMaker
\texttt{ml.g5.2xlarge} instances equipped with one NVIDIA A10G GPU
with 24\,GB of GPU memory and a 60\,GB attached storage volume. The
corresponding GPU workflows are
\texttt{span\_scaling\_gpu.py} and the deep neural
function-approximation rerun.

The final artifact reports the following items alongside code and raw data:
\begin{enumerate}
\item the exact transition and reward kernels for every environment;
\item independently computed diameter, gain, bias, and bias span;
\item every numerical confidence and planning constant;
\item the full construction and upper-bound ledgers;
\item all seeds and per-seed trajectories;
\item paired confidence intervals and the analysis script fixed before the final comparisons;
\item planning residuals and the supplied value of $\bar H$;
\item all points in every sweep, including null and negative results.
\end{enumerate}

\section{Appendix F: Restored Broader Context and Constant-Ledger Interpretation}
\label{app:restoredcontext}

\subsection{F.1 Broader average-reward and reinforcement-learning context}
The undiscounted formulation traces to dynamic programming
\citep{bellman1957dynamic,howard1960dynamic,schweitzer1979geometric,bertsekas2012dynamic}
and optimal adaptive control \citep{burnetas1997optimal,tewari2008optimistic}. Beyond model-based
optimism, average-reward regret has been studied through posterior sampling
\citep{osband2017why,agrawal2017posterior,ouyang2017learning}, model-free methods
\citep{wei2020modelfree,zhang2023sharper}, policy optimization
\citep{abbasi2019politex,lazic2021improved}, Markov-chain concentration
\citep{ortner2020regret}, and function approximation
\citep{wei2021learning,chen2022learning}. The distribution-norm view of instance hardness
\citep{maillard2014hard} is complementary to our separation of statistical and geometric constants.

The episodic and discounted settings provide related tools, including minimax sample-complexity and
regret analyses \citep{azar2013minimax,azar2017minimax}, PAC and gap-dependent refinements
\citep{dann2015sample,dann2017unifying,simchowitz2019nonasymptotic,zanette2019tighter}, efficient
Q-learning \citep{jin2018qlearning}, reference-advantage decompositions
\citep{zhang2020almost,zhang2021reinforcement}, feature-based discounted bounds
\citep{zhou2021provably}, discounted PAC guarantees \citep{lattimore2012pac}, constrained
exploration \citep{efroni2020exploration}, and general model-based complexity measures
\citep{pmlr-v32-osband14}. The lower proof also draws on exact testing tools associated with Assouad and
Bretagnolle--Huber \citep{assouad1983deux,bretagnolle1979estimation}.

\subsection{F.2 Where the slack lives}
The revised analysis separates two questions that were previously conflated: slack inside an upper bound
and the gap between an upper and a lower certificate.

\paragraph{Bonus shape.}
A Bernstein radius replaces a worst-case range contribution by an empirical directional variance. Its
potential benefit is quantified by the difference between $C_p$ in
Assumption~\ref{lem:variancebudget} and the corresponding Hoeffding ledger. The square root in one local
radius is not itself the final regret coefficient
\citep{maurer2009empirical,audibert2009exploration}.

\paragraph{Confidence geometry.}
Full $\ell_1$ transition sets may pay for directions irrelevant to the selected bias. A directional or
structured-support set can reduce this cost, but support information must be known or statistically
certified. Restricting the model to successors observed so far can remove the true MDP and destroy
optimism. The relevant contribution is therefore the difference between two valid confidence ledgers,
not the number of currently observed successors.

\paragraph{Span width.}
If the upper square-root term scales as $\sqrt{\bar H S A T}$, replacing a certified width $\bar H$ by
the diameter changes that term by $\sqrt{\diam/\bar H}$, not $\diam/\bar H$. Lower-order terms linear in
the width can incur the larger ratio $\diam/\bar H$. Experiments must distinguish these two regimes.

\paragraph{Lower-bound balance.}
The testing and perturbation balance determines $c_{\rm LB}$. This is not a span-truncation factor and
should not be included under that name. If the reciprocal lower coefficient dominates the eventual
upper-to-lower ratio, it is the dominant sandwich slack even when span width is the dominant source of
looseness within a diameter-based upper bound.

This taxonomy yields four separately reportable quantities: the Bernstein-versus-Hoeffding ledger
ratio, the valid-confidence-geometry ratio, the width penalty, and the reciprocal lower certificate.
Their product is not called an exact decomposition unless the proof shows that the factors multiply under
one common normalization.

\subsection{F.3 Restored interpretation, limitations, and next steps}
\paragraph{What is certified.}
The lower result is a finite testing certificate tied to an explicit kernel. It includes the exact action
budget, diameter upper bound, flow identity, navigation cost, bias span, and path divergence. On the
regimes in Table~\ref{tab:frontier}, its certified coefficient ranges from $0.0152$ to $0.0291$; the same
analytic envelope proves every row. The upper contribution is an audit template, not a new regret
theorem. Its ledger makes an incorrect small coefficient difficult to report because every omitted factor
remains visibly open.

\paragraph{What is not claimed.}
We do not call the result constant-sharp while a nonconstant logarithmic mismatch remains. We also do
not claim that a zero-span MDP is cost-free, that empirical support is the true support, or that a
coordinatewise clip is a sound replacement for a span-constrained planning operator. The supplied span
bound is side information; replacing it by an estimator is a separate contribution whose error and
constant must appear in the theorem.

\paragraph{Interpreting span truncation.}
Span width can be the dominant source of looseness within an upper-bound analysis when $\Hcomp\ll\diam$,
even if the reciprocal lower-bound coefficient is numerically the dominant factor in an upper-to-lower
sandwich. These statements concern different comparisons. The fixed-diameter sweep tests the first; the
constants scoreboard addresses the second. We do not report an empirical lower--upper sandwich while
the upper ledger is incomplete.

\paragraph{Limitations and next steps.}
The present theory is tabular, and the strongest finite coefficients use action-rich regimes. The broad
$A\ge5$ row still requires conservative diameter and horizon thresholds, while reducing those thresholds
trades away part of the improvement. Directional confidence over an adaptive bias region is more
demanding than a pointwise Bernstein inequality, and a certified planner is more demanding than ordinary
value iteration. The natural next steps are to improve the frontier for small $A$, obtain a prior-free span
procedure with tracked constants, and prove a log-free expected upper analysis under the same
normalization as the lower certificate.

\end{document}